\documentclass[11pt]{article}

\usepackage[margin=1in]{geometry}
\usepackage{amsmath,amssymb,amsthm,mathtools,bm}
\usepackage{microtype}
\usepackage{hyperref}
\usepackage[notrig]{physics}
\usepackage[nameinlink,capitalize,noabbrev]{cleveref}
\usepackage{booktabs}
\usepackage{array}
\usepackage{tabularx}
\usepackage{float}
\usepackage{enumitem}
\usepackage[round,authoryear]{natbib}
\usepackage{authblk}

\allowdisplaybreaks
\newtheorem{theorem}{Theorem}[section]
\newtheorem{proposition}[theorem]{Proposition}
\newtheorem{lemma}[theorem]{Lemma}

\theoremstyle{definition}
\newtheorem{assumption}[theorem]{Assumption}
\newtheorem{example}[theorem]{Example}
\newtheorem{onlinealgorithm}{Algorithm}
\theoremstyle{remark}
\newtheorem{remark}[theorem]{Remark}

\crefname{assumption}{Assumption}{Assumptions}
\crefname{onlinealgorithm}{Algorithm}{Algorithms}

\DeclareMathOperator*{\argmin}{\arg\!\min}
\newcommand{\ind}{\mathbf 1}
\newcommand{\cE}{\mathcal E}

\newcommand{\cG}{\mathcal G}
\newcommand{\cH}{\mathcal H}
\newcommand{\cK}{\mathcal K}
\newcommand{\cM}{\mathcal M}
\newcommand{\cX}{\mathcal X}
\newcommand{\regret}{\mathsf{Regret}}
\newcommand{\fR}{\mathfrak R}
\newcommand{\sd}{\mathsf d}
\newcommand{\sm}{\mathsf m}
\renewcommand{\sp}{\mathsf p}
\newcommand{\sK}{\mathsf K}
\newcommand{\sP}{\mathsf P}
\newcommand{\E}{\mathbb E}
\newcommand{\bH}{\mathbb H}
\newcommand{\bN}{\mathbb N}
\newcommand{\bP}{\mathbb P}
\newcommand{\hbP}{\widehat{\mathbb{P}}^{\rm pred}}
\newcommand{\hsp}{\widehat{\mathsf{p}}^{\rm pred}}
\newcommand{\N}{\texttt{\textup{N}}}
\newcommand{\KL}{\mathrm{KL}}
\newcommand{\W}{W_2}
\newcommand{\dist}{\mathsf d}
\renewcommand{\d}{\mathrm d}
\newcommand{\e}{\mathrm e}
\newcommand{\R}{\mathbb R}
\newcommand{\fa}{\mathfrak a}
\newcommand{\fs}{\mathfrak s}
\newcommand{\keywords}[1]{\par\medskip\noindent\textbf{Keywords: }#1}
\newcommand{\BT}{\mathfrak B_T}
\newcommand{\Wp}{W_{2,p}}
\newcommand{\ThetaSob}{\Theta_{\fs}(B, R)}
\newcommand{\TV}{\operatorname{TV}}
\newcolumntype{Y}{>{\raggedright\arraybackslash}X}

\newcommand{\genericdel}[4]{%
  \ifcase#3\relax
  \ifx#1.\else#1\fi#4\ifx#2.\else#2\fi\or
  \bigl#1#4\bigr#2\or
  \Bigl#1#4\Bigr#2\or
  \biggl#1#4\biggr#2\or
  \Biggl#1#4\Biggr#2\else
  \left#1#4\right#2\fi
}
\newcommand{\del}[2][-1]{\genericdel(){#1}{#2}}
\newcommand{\cbr}[2][-1]{\genericdel\{\}{#1}{#2}}
\newcommand{\sbr}[2][-1]{\genericdel[]{#1}{#2}}

 \let\abs\Abs
 \let\norm\Norm
 
\newcommand{\ceil}[2][-1]{\genericdel\lceil\rceil{#1}{#2}}

\newcommand{\opnorm}[2][0]{\norm[{#1}]{{#2}}_{\textup{op}}}
\newcommand{\inn}[2]{\left\langle {#1}, {#2} \right\rangle}

\title{Fast rates in Bayesian online learning with approximate posteriors}
\author{Ilsang Ohn}
\affil{Department of Statistics, Inha University}
\date{}

\begin{document}
\maketitle

\begin{abstract}
Exact Bayes prediction enjoys fast predictive regret guarantees, but exact posterior updating or representation may be too costly for online use. We study when these statistical guarantees are preserved by computational approximations. We show that the cumulative price of posterior approximation can be governed by the interaction between the contraction radius of the exact Gibbs posterior and the Wasserstein distance between the approximate and exact posteriors. Our general theorem shows that whenever exact Bayes prediction achieves a fast regret bound, any approximate posterior method that tracks the exact posterior with sufficient accuracy inherits the same fast regret, up to an additive term determined by the approximation error. Three online learning examples are developed. For linear models with strongly convex regularized losses, a projected Langevin algorithm yields an approximate posterior that achieves logarithmic regret. For an infinite-dimensional canonical exponential family sequence model over a Sobolev ellipsoid, a prior-preserving truncation method attains the minimax predictive regret rate with sublinear memory and constant update cost per observation. For random-design Gaussian process (GP) regression, a sparse variational posterior with inducing variables achieves the same predictive regret rate as the exact GP, but at substantially lower computational cost.
\end{abstract}


\section{Introduction}\label{sec:introduction}

Bayesian online prediction maintains a distribution over parameters and uses it as the mixing law for the next prediction: after $t$ observations, the posterior $\Pi_t$ induces the predictive mixture for round $t+1$. Under logarithmic loss such as negative log-likelihood, this construction has an exceptional cumulative structure: the one-step predictive normalizers telescope into a marginal likelihood, so that cumulative predictive regret can be analyzed through a single integrated likelihood ratio rather than by controlling the prediction error separately at every round.  In regular parametric models, this connection leads to the classical logarithmic redundancy and regret behavior of Bayes mixtures. The information-theoretic analysis of  \citet{clarke1990bayes} establishes the asymptotic redundancy of Bayesian mixtures in regular parametric families, while \citep{xie2000minimax} develop closely related asymptotic minimax regret results for universal coding, gambling, and sequential prediction. 

The same phenomenon has a natural interpretation from online learning. Under log loss, Bayesian updating is an instance of exponential weighting, and logarithmic loss is mixable, so a mixture predictor can compete sharply with a fixed expert or parameter chosen in hindsight \citep{vovk1990aggregating,cesabianchi2006prediction}. \citet{kakade2004online} derive online bounds for Bayesian algorithms and show how Bayesian predictors can be analyzed as online learning procedures without stochastic assumptions on the outcome sequence. More generally, exponential-weights methods admit regret guarantees under a variety of mixable and exp-concave losses, with fast rates under suitable curvature or stochastic conditions \citep{hazan2007logarithmic,vanerven2015fast,vanderhoeven2018manyfaces,jun2026adaptive}. The resulting viewpoint is broader than ordinary Bayes: posterior-like exponential weights can be formed from generic losses, and the same variational structure underlying Bayesian marginal likelihoods can be used to compare a mixture predictor with a fixed parameter or comparator distribution.
 
The above literature, however, primarily analyzes the statistical regret of the exact updating or weighting rule. The distribution used for prediction is assumed to be available at every round, and the computational error incurred in representing, sampling from, or optimizing that distribution is not priced separately. This distinction is immaterial when the posterior or exponential weights can be computed exactly, but becomes important in modern Bayesian models where each online update is itself approximate. Posterior approximation, which is statistically adequate at a fixed terminal sample size, need not be adequate for online prediction, because the approximation is queried repeatedly along the entire posterior path. Consequently, terminal contraction of an approximate posterior toward the truth does not by itself control the cumulative predictive loss. What is needed is a way to quantify how perturbations of the posterior at each round are translated into perturbations of the predictive mixture, and how this sensitivity changes as the exact posterior concentrates. Because of this challenge, to our knowledge, there is no general result that converts posterior approximation error into a cumulative predictive regret penalty while retaining the fast rates of exact Bayesian prediction. Aside from the specific setting of variational Bayes with a Gaussian variational family \citep{cherief2019onlinevi}, the known results only provide slow regret bounds of the order $O(\sqrt{T})$ for horizon $T$ \citep[e.g.,][]{cherief2019onlinevi,alquier2021non}.

This raises the question addressed in this paper: do the fast rates of Bayesian online prediction survive computation? Concretely, is there an implementable online learner, with explicit per-round operations, whose expected cumulative predictive regret is logarithmic in regular parametric models and minimax optimal in nonparametric ones? 
 
We answer this question affirmatively. We establish a theoretical framework to derive fast predictive regret guarantees for general computational Bayesian online learning methods. Let $\Pi_t$ be the exact Gibbs or standard posterior after $t$ observations and let $Q_t$ be an approximate posterior distribution attained by a computationally feasible algorithm. We derive an upper bound on the expected difference between the predictive regrets of $\Pi_t$ and $Q_t$, in terms of $\epsilon_t$, the contraction rate of $\Pi_t$ toward a common stochastic target, and $\alpha_t$, the Wasserstein distance between $Q_t$ and $\Pi_t$.  The novelty of our result lies in showing that posterior approximation need not be controlled at a fixed first-order accuracy throughout the online path. We devise a second-order predictive stability condition, under which the effect of approximation error is automatically attenuated by posterior contraction, so the cumulative approximation cost scales with the interaction $\epsilon_t\alpha_t$, rather than with $\alpha_t$ alone.  

We instantiate this principle in three regimes. The first is a finite-dimensional generalized Bayesian model with predictable bounded covariates and a globally strongly convex regularized loss. Strong convexity gives $O(t^{-1/2})$ contraction of the exact Gibbs posterior, while local prior mass gives $O(\log T)$ expected regret for exact prediction. We approximate each constrained posterior with a projected Moreau--Yosida Langevin construction. We derive a squared Wasserstein tracking error of order $t^{-1}$ on the compact parameter space. This results in an $O(\log T)$ cumulative approximation penalty, preserving the logarithmic rate of exact Gibbs prediction. 

The second regime is an infinite-dimensional canonical exponential family sequence model. Here the computational constraint is representation rather than finite-dimensional sampling. The exact posterior is a countable product of one-dimensional conjugate posteriors. A streaming implementation may retain the sufficient statistics of every coordinate observed so far, but its memory can grow linearly with the horizon. We instead update only the first $m$ coordinates and preserve the original prior on the unresolved tail. This construction approximates one fixed infinite-dimensional Bayesian model rather than replacing it with a sequence of truncated-prior models. Over a Sobolev ellipsoid, the contraction radius, exact Gibbs regret, and truncation error are all explicit. Choosing $m$ at the effective-dimension scale yields the minimax cumulative predictive regret rate with sublinear memory and constant state-update work per observation.

The third regime is random-design Gaussian process (GP) regression. Here the statistical question is whether a sparse variational representation preserves the predictive performance of the full GP posterior. We fix the first $J$ population covariance operator eigenfunctions as interdomain inducing variables. Under Gaussian noise, the resulting sparse variational posterior is represented by finite-dimensional sufficient statistics and can be updated recursively. When $J$ is greater than a certain sublinear threshold in $T$, its Wasserstein distance from the exact posterior is of the same order as the exact posterior contraction radius, and its cumulative predictive regret has the same nonparametric order as exact Bayes. This example also shows that vanishing variational KL divergence is not necessary for preserving statistical performance: the relevant discrepancy is its displacement relative to the shrinking posterior geometry.

The three applications make different computational choices. In the finite-dimensional example, approximation arises from incomplete sampling; in the sequence model, it arises from truncating an infinite representation; and in the GP example, it arises from variational sparse compression. Despite these different mechanisms, all three are analyzed through the same pair of quantities: the statistical radius of exact Bayes and the distance of the computational approximation from exact Bayes. The resulting framework makes the required computational accuracy explicit and exposes where model-specific work is unavoidable.

\subsection{Related work}
 
\paragraph{Streaming and approximate Bayesian inference}
A large algorithmic literature develops tractable approximations to sequential Bayesian updating. Examples include streaming variational Bayes \citep{broderick2013streaming}, sequential Monte Carlo \citep{delmoral2006smc}, particle mirror descent \citep{dai2016particle}, online variational inference with generalization guarantees \citep{cherief2019onlinevi}, recursive optimization-based Bayesian rules \citep{khan2023bayesian, jones2024bong}, online sampling from log-concave sequences \citep{lee2019online}, and particle-based approximations of moving posteriors \citep{yang2023particle}. \citet{cherief2019onlinevi,alquier2021non} establish generalization guarantees for several online variational procedures, providing one of the closest theoretical connections between approximate Bayesian computation and online learning. By contrast, \citep{yang2023particle} study Wasserstein tracking of a dynamically changing posterior through a distributional dynamic-regret analysis.

Recent work has also begun to characterize the frequentist accuracy of sequentially approximated posteriors themselves. \citet{lee2026online} stablish an online Bernstein--von Mises theorem for sequential variational updating with mini-batches, showing that accumulated approximation error can become asymptotically negligible and that the resulting terminal posterior can be asymptotically equivalent to the full posterior, while \citet{lee2026bayesian} treat the more stringent one-pass regime and construct a warm-started Bayesian online procedure that attains an optimal convergence rate and a Bernstein--von Mises limit without diverging mini-batch sizes. These results concern the inferential fidelity of the sequentially constructed posterior, especially at the end of the data stream. In contrast, our criterion is prequential: approximation error is charged whenever the approximate posterior is used for prediction, and hence must be controlled along the entire posterior path rather than only at the terminal sample size.

\paragraph{Statistical theory for approximate posteriors}

A complementary literature studies whether computational approximations retain the frequentist concentration properties of exact Bayes. Variational and tempered posteriors have been shown to attain optimal or near-optimal contraction rates in a range of parametric, high-dimensional, and nonparametric models \citep{alquier2016properties,alquier2020concentration,zhang2020convergence,ohn2024adaptive}. Such results typically compare an approximate posterior directly with the true parameter or data-generating distribution. This comparison is not sufficient for cumulative prediction: two posterior sequences can contract toward the same target at the same statistical rate while remaining separated enough to generate different predictive mixtures at every round. Our analysis therefore keeps two distinct quantities of the statistical radius of the exact posterior and the computational distance between the approximate and exact posteriors. This separation is what allows a posterior approximation guarantee to be converted into a fast predictive regret guarantee.
 
\paragraph{Constrained log-concave sampling}
Nonasymptotic sampling theory provides explicit distributional approximation guarantees for log-concave targets. For smooth log-concave distributions on Euclidean space, unadjusted Langevin algorithms have quantitative convergence guarantees \cite{dalalyan2017theoretical, durmus2017nonasymptotic}. For compactly supported targets, projected Langevin Monte Carlo \citep{bubeck2018projected}  and Moreau--Yosida regularized Langevin methods \citep{brosse2017compact}  provide tractable schemes with total-variation or Wasserstein guarantees with explicit schedules, while proximal samplers achieve sharper complexity under alternative oracle assumptions \citep{lee2021structured}.  Online sampling from  sequence of log-concave distributions has also been studied directly \citep{lee2019online}.
 
\paragraph{Computational-statistical tradeoffs and low-rank posteriors.}
The principle that computation may be deliberately stopped or compressed once its error falls below statistical resolution appears in several batch problems. Examples include convex-relaxation hierarchies \citep{chandrasekaran2013computational}, Nystr\"om subsampling for kernel regression \citep{rudi2015less}, and early stopping for nonparametric regression \citep{raskutti2014early},
Low-rank approximation of Gaussian inverse problems similarly retains directions that are most informative relative to the prior  \citep{spantini2015optimal}, while finite-dimensional computation in infinite-dimensional Bayesian models is often imposed through sieve or truncated priors \citep{zhao2000bayesian}. Our results give a sequential Bayesian counterpart to this computational--statistical principle. In particular, in the sequence model approximation in \cref{sec:sequence}, the Bayesian model remains infinite-dimensional, whereas the online algorithm updates only a finite number of coordinates, leaving the unresolved tail at its original prior.
 
\paragraph{Sparse variational Gaussian processes.}
Sparse variational GP regression represents a process posterior through finitely many inducing variables \citep{titsias2009variational,matthews2016sparse}. Their approximation accuracy and the required growth of the inducing budget have been studied extensively. \citet{burt2019sparsegp,burt2020convergence} derive quantitative KL approximation guarantees and sufficient inducing dimensions, while \citet{nieman2022contraction,nieman2023uncertainty} establish posterior contraction and uncertainty quantification results for several inducing schemes, including population-spectral constructions. More recent work develops adaptive sparse variational procedures that retain minimax contraction behavior under hyperparameter selection \citep{nieman2025adaptive}. On the algorithmic side, variational Fourier features and streaming sparse GP methods provide scalable spectral or sequential representations \citep{hensman2018variational,bui2017streaming}. Our concern is different from contraction of the sparse posterior around the truth. We compare the sparse and exact predictive mixtures over the entire online sequence and ask how large the inducing representation must be to preserve the cumulative predictive regret order of full GP Bayes.
 
\subsection{Contributions and organization}
 
The paper makes four contributions. First, \cref{sec:regret} establishes a general predictive comparison theorem. It decomposes the expected regret of an approximate posterior into the regret of the exact Gibbs posterior and a posterior approximation penalty. Second, \cref{sec:langevin} proves a logarithmic expected regret for an implementable projected Moreau--Yosida Langevin online learner over globally log-concave linear-predictor Gibbs models. Third, \cref{sec:sequence} proves that a prior-preserving truncated Bayesian algorithm attains the minimax predictive regret rate for an exponential family sequence model with sublinear memory and constant-time updates. Fourth, \cref{sec:gp} develops an online population-spectral sparse variational GP and proves that an effective-dimension representation preserves the nonparametric predictive regret order of exact Bayes. The analysis identifies posterior uncertainty, rather than a vanishing variational objective gap, as the statistical scale that determines the required inducing dimension. \cref{sec:discussion} concludes the paper and discusses limitations. Detailed proofs are collected in the appendices.

\subsection{Notation}

For a set $A$, let $\ind_A(\cdot)$ denote the indicator function of $A$ such that $\ind_A(x)=1$ if $x\in A$ and $\ind_A(x)=0$ otherwise. For a vector $x$, let $\|x\|$ denote its $L^2$ norm. We write $\operatorname{diam}(A):=\sup_{x,y\in A}\|x-y\|$ for a subset $A$ of $\R^d$.
We write $a\lesssim b$ or $b\gtrsim a$ if there is a constant $C>0$, not relevant to the main argument, such that $a\le Cb$. We write $a\asymp b$ if both $a\gtrsim b$ and $a\lesssim b$ hold.
For two probability measures $P$ and $Q$, $\KL(P\|Q)$ denotes the Kullback--Leibler (KL) divergence defined as $\KL(P\|Q):=\int\log(\d P/\d Q)\d P$ if $P\ll Q$, and $\KL(P\|Q):=\infty$ otherwise.

\section{Main results}\label{sec:regret}

\subsection{Statistical setup}

We work on a filtered probability space. At round $t\in\bN$, we observe a pair of the input or context variable $X_t$ and the response $Y_t$.  The sigma-field $\cG_{t-1}$ contains history before the round-$t$. In particular it may contain predictable side information such as the current covariate or context, as well as all computational randomness generated before round $t$. We allow the round $t$ loss map
    \begin{align*}
    (\theta,(x,y))\mapsto \ell_t(\theta,(x,y))
    \end{align*}
to be $\cG_{t-1}$-measurable as a random function of $(\theta,(x,y))$. Thus fixed-design, predictable-design, independent non-identically distributed, and i.i.d. models are all special cases. We do not consider adversarial outcomes; the results below require a common stochastic target through conditional moment and curvature conditions that are verified in the applications.

Let $(\Theta,\dist)$ be a Polish metric space. We fix a prior $\Pi_0$ and inverse temperature $\eta>0$. The exact Gibbs posterior is
     \begin{equation}\label{eq:gibbs}
     \Pi_t(\d\theta)
     =\frac{\e^{-\eta L_t(\theta)}\Pi_0(\d\theta)}{\int \e^{-\eta L_t(u)}\Pi_0(\d u)},
     \end{equation}
where we write
    \begin{align*}
    L_t(\theta):=\sum_{s=1}^t\ell_s(\theta; (X_s, Y_s)).
    \end{align*}
For a probability measure $Q$ on $\Theta$, define the round-$t$ \textit{predictive loss} by
     \begin{equation}\label{eq:mixloss}
      \sm_{t}(Q; (x,y))
      :=-\frac1\eta\log\int \e^{-\eta\ell_t(\theta; (x,y))}Q(\d\theta).
     \end{equation}
When $\eta=1$ and the loss function is the negative log density of $Y_t$ conditional on $X_t$ and $\cG_{t-1}$, this is ordinary Bayesian predictive log loss. Let $\theta^\star\in\Theta$ be a fixed stochastic target. A computational procedure maintains a random probability measure $Q_t$ that is measurable with respect to the information available after round $t$ and before $t+1$. We assume throughout that every Gibbs normalizing constant and predictive loss integral displayed above is almost surely finite and strictly positive.

We set $Q_0=\Pi_0$ and define the \textit{predictive regret} of a sequence of online distributions $(Q_t)_{t=0}^{T-1}:=(Q_0,Q_1,\dots, Q_{T-1})$ as
     \begin{equation}\label{eq:regret}
     \regret((Q_t)_{t=0}^{T-1})
     :=\sum_{t=1}^T\sbr{\sm_{t}(Q_{t-1}; (X_t, Y_t))-\ell_t(\theta^\star;(X_t, Y_t))}
     \end{equation}
and denote its expectation as
    \begin{align*}
      \fR\del[1]{(Q_t)_{t=0}^{T-1}}
      :=\E\sbr{\regret((Q_t)_{t=0}^{T-1})}
    \end{align*}
Let $\W$ denote the quadratic Wasserstein distance induced by $\dist$, and define
    \begin{equation}\label{eq:eps-a}
    \varepsilon_t^2:=\E [\W^2(\Pi_t,\delta_{\theta^\star})], \qquad
    \alpha_t^2:=\E [\W^2(Q_t,\Pi_t)].
    \end{equation}
The two radii in the above display deliberately measure different objects. The statistical radius $\varepsilon_t$ is a property of the exact posterior and the data-generating process. The computational radius $\alpha_t$ is a property of the approximation algorithm relative to the exact update. Keeping them separate will allow us to reuse the same statistical analysis for computational procedures with different error mechanisms.

\subsection{Regret of exact Gibbs posteriors}

As we mentioned in the introduction, the exact Gibbs posterior admits a telescoping identity, which can lead to fast regret rates. To be concrete, let 
     \begin{align*}
     Z_t^\Pi:=\int \e^{-\eta \sum_{s=1}^t\ell_s(\theta; (X_s, Y_s))}\Pi_0(\d\theta).
     \end{align*}
Then the Gibbs recursion gives
    \begin{align*}
    \int \e^{-\eta\ell_t(\theta; (X_t, Y_t))}\Pi_{t-1}(\d\theta)
    =\frac{Z_t^\Pi}{Z_{t-1}^\Pi}.
    \end{align*}
Thus the predictive losses telescope as
    \begin{align*}
    \sum_{t=1}^T\sm_{t}(\Pi_{t-1}; (X_t, Y_t))
    = -\frac1\eta\sum_{t=1}^T\{\log Z_{t}^\Pi-\log Z_{t-1}^\Pi\}
    =-\frac1\eta\log Z_{T}^\Pi.
    \end{align*}
Subtracting $\sum_t\ell_t(\theta^\star; (X_t, Y_t)$ yields 
     \begin{equation}\label{eq:telescoping}
    \regret\del{(\Pi_t)_{t=0}^{T-1}}
    =-\frac1\eta
    \log\int\exp\sbr[3]{-\eta\sum_{t=1}^T\{\ell_t(\theta; (X_t, Y_t)-\ell_t(\theta^\star; (X_t, Y_t)\}}\Pi_0(\d\theta).
    \end{equation}
To bound the expectation of the right-hand side, the following proposition can be used.

\begin{proposition}[Gibbs regret bound]
\label[proposition]{prop:gibbs-benchmark}
%For every deterministic probability measure $\rho\ll\Pi_0$,
%\begin{equation}\label{eq:variational-benchmark}
%\fR\del[1]{(\Pi_t)_{t=0}^{T-1}}\le
%\sum_{t=1}^T\int\E\{\ell_t(\theta,Z_t)-\ell_t(\theta^\star,Z_t)\}\rho(\d\theta)
%+\frac1\eta\KL(\rho\|\Pi_0).
%\end{equation}
Define
\begin{equation}\label{eq:benchmark-def}
\BT
:= \inf_{\rho\ll\Pi_0}
\left[
\sum_{t=1}^T
\int\E[\ell_t(\theta; (X_t, Y_t))-\ell_t(\theta^\star; (X_t, Y_t))]\rho(\d\theta)
+\frac1\eta\KL(\rho\|\Pi_0)
\right].
\end{equation}
Then $\fR\del[1]{(\Pi_t)_{t=0}^{T-1}}\le\BT$.
\end{proposition}

\subsection{Main  theorem}

For a deterministic probability measure $Q$ on $\Theta$, define the \textit{conditional predictive risk} as
    \begin{equation}\label{eq:conditional-M}
    \cM_t(Q):=\E[ \sm_{t}(Q; (X_t,Y_t))|\cG_{t-1}].
    \end{equation}
For a $\cG_{t-1}$-measurable random measure $Q$, $\cM_t(Q)$ denotes the evaluation of this random functional at $Q$. Then the expected regret of $(Q_t)_{t=0}^{T-1}$ can be expressed as
     \begin{align*}
      \fR\del[1]{(Q_t)_{t=0}^{T-1}}
      &=\sum_{t=1}^{T}\E\sbr{\E\sbr{ \sm_{t}(Q_{t-1}; (X_t,Y_t))-\ell_t(\theta^\star;(X_t,Y_t))|\cG_{t-1}}}\\
      &=\sum_{t=1}^{T}\E\sbr{\cM_t(Q_{t-1})-\cM_t(\delta_{\theta^\star})}.
      \end{align*}
By a similar argument,  we also have
    \begin{align}
    \fR\del[1]{(Q_t)_{t=0}^{T-1}}
    - \fR\del[1]{(\Pi_t)_{t=0}^{T-1}}
    &=\sum_{t=2}^{T}\E\sbr{\cM_t(Q_{t-1})-\cM_t(\Pi_{t-1})},
  \label{eq:regret-conditional-risk}
 \end{align} 
where the comparator term $\cM_t(\delta_{\theta^\star})$ cancels when the two regrets are subtracted. 

\begin{assumption}[Conditional second-order predictive stability]
\label[assumption]{ass:predictive-stability}
There is an absolute constant $L_{\cM}>0$ such that
    \begin{equation}\label{eq:pred-stability}
    \cM_t(Q)-\cM_t(P)
    \le L_{\cM}\cbr{\W(P,\delta_{\theta^\star})\W(P,Q)+ \W^2(P,Q) }
    \end{equation}
almost surely for every $t\in\bN$ and for every pair of probability measures $P,Q$.
\end{assumption}

\cref{ass:predictive-stability} is a \textit{second-order} continuity condition on the conditional predictive risk: both terms on the right-hand side of \eqref{eq:pred-stability} are products. A first-order Lipschitz condition $\cM_t(Q)-\cM_t(P) \lesssim \W(P,Q)$ would be easier to verify but yields the cumulative penalty $\sum_{t=1}^{T-1}\alpha_t$, requiring increasingly accurate computation regardless of statistical concentration. Under second-order stability, however, the same perturbation becomes progressively less consequential as the exact posterior contracts. Thus posterior concentration plays a dual role: it improves statistical prediction and simultaneously reduces the sensitivity of prediction to computational approximation.  The quadratic term $\alpha_t^2$ is the residual cost of the approximation itself.

%which is strictly weaker whenever  $\varepsilon_t\to0$. When the target $\theta^\star$ is a stationary point of the risk, then this second-order bound can be  available; see \cref{prop:primitive-stability} given in the next subsection.

\begin{theorem}[Predictive regret under posterior approximation]\label{thm:main-transfer}
Under \cref{ass:predictive-stability},
    \begin{equation}\label{eq:main-transfer}
    \fR\del[1]{(Q_t)_{t=0}^{T-1}}
    \le
    \fR\del[1]{(\Pi_t)_{t=0}^{T-1}}
    +L_{\cM} \sum_{t=1}^{T-1}\cbr{\varepsilon_t\alpha_t+\alpha_t^2}.
    \end{equation}
\end{theorem}

%The above theorem implies that, if the series on the right-hand side is comparable to the regret of the exact Gibbs posterior, then that of the approximate posterior enjoys. If $\varepsilon_t\lesssim t^{-\beta}$ and $\alpha_t\lesssim t^{-\gamma}$, then the approximation penalty is controlled by
% \begin{align*}  
% \sum_{t=1}^{T-1}t^{-(\beta+\gamma)}+\sum_{t=1}^{T-1}t^{-2\gamma}
 %\end{align*}
%which is bounded by $O(\log T)$ when $\beta+\gamma\ge 1$ and $\gamma\ge 1/2$, where the $O(\log T)$
%\cref{sec:langevin} verifies this requirement for Gibbs models with linear predictors, while \cref{sec:sequence} specializes the same framework to an i.i.d. random-context sequence model.

\subsection{Sufficient conditions for second-order predictive stability}\label{sec:sufficient-stability}

\cref{ass:predictive-stability} is stated at the level of predictive distributions. The next result reduces it to conditions on the underlying loss. Smoothness controls the second-order interpolation remainder, the exponential envelope controls the change of measure induced by the predictive normalizer, and conditional stationarity removes a nonvanishing first-order term at $\theta^\star$. Note that the condition is stated with random envelopes so that it also applies coordinatewise to the exponential family sequence model in \cref{sec:sequence}.

\begin{theorem}[Sufficient conditions of  \cref{ass:predictive-stability}]
\label{prop:primitive-stability}
Let $\Theta\subset\R^d$ be convex and compact. Suppose that, for every $t$, almost surely, the map $\ell_t(\cdot;(x,y))$ is continuously differentiable on a neighborhood of $\Theta$ for almost every conditional realization $(x,y)$ of $(X_t,Y_t)$, and there are nonnegative $\cG_t$-measurable envelopes $H_t,G_t,D_t$ such that
     \begin{align}
     \|\nabla\ell_t(\theta;(X_t,Y_t))-\nabla\ell_t(u;(X_t,Y_t))\|
     &\le H_t\|\theta-u\|,\qquad \theta,u\in\Theta, \label{eq:loss_grad_lip}\\
     \sup_{\theta\in\Theta}\|\nabla\ell_t(\theta;(X_t,Y_t))\|
     &\le G_t,\label{eq:loss_grad_bound}\\
     \sup_{\theta,u\in\Theta}|\ell_t(\theta;(X_t,Y_t))-\ell_t(u;(X_t,Y_t))|
     &\le D_t.\label{eq:loss_lip}
    \end{align}
Assume that there exists an absolute constant $C_{\rm env}>0$ such that
    \begin{equation}\label{eq:primitive-envelope}
    \E\sbr{\e^{\eta D_t}(H_t+\eta G_t^2)\mid\cG_{t-1}}
    \le C_{\rm env}\qquad
    \text{a.s. for every }t\in\bN.
    \end{equation}
Moreover, assume that there exists $\theta^\star\in\operatorname{int}(\Theta) $ such that
    \begin{equation}\label{eq:conditional-stationarity}
    \E[\nabla\ell_t(\theta^\star;(X_t, Y_t))\mid\cG_{t-1}]=0
    \qquad\text{a.s.}
    \end{equation}
Then \cref{ass:predictive-stability} holds for all probability measures supported on $\Theta$, with a deterministic constant depending only on $C_{\rm env}$, $\eta$, and $\operatorname{diam}(\Theta)$.
\end{theorem}

\section{Moreau--Yosida Langevin online Bayes}\label{sec:langevin}
In this section, we apply our general result to a regular parametric problem with a strongly convex loss function. It turns out that strong convexity yields $\varepsilon_t\asymp t^{-1/2}$. \cref{thm:main-transfer} then shows that any posterior approximation method with $\alpha_t\lesssim t^{-1/2}$ has an approximation penalty of order $\log T$. The sampling analysis below is used precisely to reach this threshold.

\subsection{Linear Gibbs models}\label{sec:linear-model}

We consider the online data generating process with predictable bounded design.

\begin{assumption}[Predictable design]
\label[assumption]{ass:design}
At round $t$, a covariate $X_t\in\R^d$ is $\cG_{t-1}$-measurable and satisfies $\|X_t\|\le K_x$ for some absolute constant $K_x>0$ almost surely.  
\end{assumption}

We use the regularized loss
     \begin{equation}\label{eq:regularized-linear-loss}
     \ell_{\lambda,t}(\theta;Y_t)
     =\phi(Y_t,X_t^\top\theta) + \frac\lambda2\|\theta\|^2,
     \qquad \lambda>0.
     \end{equation}
Note that the predictable covariate is part of the round-$t$ loss function rather than the random covariate in the abstract setup of \cref{sec:regret}. We impose some regularity conditions on the loss function and the parameter space.

\begin{assumption}[Loss and parameter space]
\label[assumption]{ass:linear-loss}
The parameter space $\Theta$ is a compact convex set with nonempty interior, which satisfies $B(0,r_\Theta)\subseteq\Theta\subseteq B(0,R_\Theta)$
for some fixed $0<r_\Theta\le R_\Theta<\infty$. For every $y$, the function $u\mapsto\phi(y,u)$ is convex and continuously differentiable. There are absolute constants $B_1>0$ and $L_\phi>0$ such that
    \begin{align*}
    |\partial_u\phi(y,u)|\le B_1,
    \qquad
    |\partial_u\phi(y,u)-\partial_u\phi(y,v)|\le L_\phi|u-v|
\end{align*}
for all $u,v\in\R$.
\end{assumption}

\begin{example}
We list several examples of a loss $\phi$ satisfying the above assumption.
\begin{enumerate}[label=(\roman*)]
\item square loss $\phi(y,u)=(y-u)^2$ with bounded $y$ and $u$;
\item logistic loss $\phi(y,u)=\log(1+\e^{-yu})$ for $y\in\{-1,1\}$;
\item Huber loss $\phi(y,u)=h_\delta(y-u)$ with 
     \begin{align*}
     h_\delta(z)
     =\begin{cases}
     z^2/2  & \text{ for }|z|\le \delta,\\
     \delta(|z|-\delta/2) & \text{ otherwise}.
     \end{cases}
     \end{align*}
\item smoothed hinge loss  $\phi(y,u)=s_\delta(yu)$ for $y\in\{-1,1\}$ with
       \begin{align*}
       s_\delta(z)
       =\begin{cases}
       0  & \text{ for } z>1\\
       1-z-\delta/2  & \text{ for } z<1-\delta\\
       (1-z)^2/(2\delta) & \text{ otherwise}.
       \end{cases}
       \end{align*}
\end{enumerate} 
\end{example}

As we did in our general setup, we assume that there is a fixed target parameter in the data-generating process.

\begin{assumption}[Conditional stationary point]
\label[assumption]{ass:stationarity}
There is a fixed $\theta^\star\in\operatorname{int}(\Theta)$ satisfying 
    \begin{equation}\label{eq:linear-mds}
    \E[\nabla\ell_{\lambda,t}(\theta^\star;Y_t)\mid\cG_{t-1}]=0
    \qquad\text{a.s. for every }t.
    \end{equation}
\end{assumption}

Under the assumptions we have made, the conditions of \cref{prop:primitive-stability} are satisfied and thus we can obtain the next result.

\begin{theorem}
\label{thm:linear_stability}
Under \cref{ass:design,ass:linear-loss,ass:stationarity}, \cref{ass:predictive-stability} holds.
\end{theorem}

Due to \cref{thm:linear_stability}, we can apply \cref{thm:main-transfer}. We then establish upper bounds on the two quantities related only to the exact Gibbs posteriors. For this purpose, we additionally require that the prior distribution is log-concave.

\begin{assumption}[Prior distribution]
\label[assumption]{ass:prior}
The prior has density
    \begin{align*}
    \Pi_0(\d\theta)    \propto \e^{-V_0(\theta)}\ind_\Theta(\theta)\d\theta,
    \end{align*}
where $V_0:\R^d\to\R$ is convex and continuously differentiable with a globally Lipschitz gradient, and the prior density is bounded below on a neighborhood of $\theta^\star$.
\end{assumption}

The exact Gibbs posterior is given by
    \begin{align*}
    \Pi_t(\d\theta) \propto \e^{-U_t(\theta)}\ind_\Theta(\theta)\d\theta
    \end{align*}
with the potential function
    \begin{align*}
    U_t(\theta) = V_0(\theta)+\eta\sum_{s=1}^t\ell_{\lambda,s}(\theta;Y_s).
    \end{align*}

\begin{theorem}[Statistical guarantees for the Gibbs posterior]
\label{thm:linear-statistical}
Under \cref{ass:design,ass:linear-loss,ass:stationarity,ass:prior}, the exact Gibbs posterior satisfies
    \begin{equation}\label{eq:linear-contraction}
    \E\sbr{\W^2(\Pi_t,\delta_{\theta^\star})}\lesssim t^{-1}
    \end{equation}
and
    \begin{equation}\label{eq:linear-gibbs}
    \fR\del[1]{(\Pi_t)_{t=0}^{T-1}}
    \lesssim \log T+1.
    \end{equation}
\end{theorem}

\subsection{Moreau--Yosida Langevin sampling}\label{sec:constrained-langevin}

For a computable approximating algorithm for the exact Gibbs posterior, we consider Moreau--Yosida unadjusted Langevin algorithm (MYULA) of \cite{brosse2017compact}. For a smoothing parameter $\varsigma>0$, define the Moreau--Yosida regularization of the hard constraint,
    \begin{equation}\label{eq:my-potential}
    U_{t,\varsigma}(\theta)
    = U_t(\theta) + \frac{1}{2\varsigma}\|\theta-\operatorname{proj}_\Theta(\theta)\|^2.
    \end{equation}
A single Moreau--Yosida unadjusted Langevin step with a step size $\gamma>0$ is
    \begin{equation}\label{eq:myula-step}
    \theta^+
    =\theta-\gamma\sbr{\nabla U_t(\theta)
    +\varsigma^{-1}\{\theta-\operatorname{proj}_\Theta(\theta)\}}
    +\sqrt{2\gamma}\xi,
    \qquad \xi\sim \N(0,I_d).
    \end{equation}
The kernel $\cK_{t,\varsigma,\gamma}$ of the Markov chain defined by \eqref{eq:myula-step} is given for $\theta\in \R^d$ and $A\subseteq \R^d$ by
    \begin{align*}
    \cK_{t,\varsigma,\gamma}(\theta, A)
    =(4\pi\gamma)^{-d/2}\int_A \exp\del{-(4\gamma)^{-1}\norm{u-\theta+\gamma \nabla U_{t,\varsigma}(\theta)}^2}\d u.
    \end{align*}
 From Theorem 2 of \citet{brosse2017compact}, we have the following result on the approximation accuracy of MYULA. 

\begin{theorem}[Approximation accuracy of MYULA]
\label{prop:myula-accuracy}
Fix $t\ge1$ and $\tilde\alpha\in(0,1)$. Under the same assumptions of \cref{thm:linear-statistical}, for every fixed initial point $\theta_0\in \Theta$, there exist
$\varsigma_t>0$, $\gamma_{t}>0$, and $N_{t}\in\bN$ such that
    \begin{equation}\label{eq:myula-tv}
    \norm[1]{\delta_{\theta_0}\cK_{t,\varsigma_t,\gamma_{t}}^{N_{t}}-\Pi_t}_{\mathrm{TV}}
    \le \tilde\alpha
\end{equation}
\end{theorem}

\begin{remark}
The chain in \eqref{eq:myula-step} can leave $\Theta$. We therefore project its final state. This does not worsen total variation distance from $\Pi_t$, because $\Pi_t$ is supported on $\Theta$ and the projection acts as the identity there.
\end{remark}

\begin{onlinealgorithm}[Online MYULA]\label{alg:proximal}
Set $Q_0=\Pi_0$ and $\tilde\alpha_t \asymp t^{-1}$. After observing $Y_t$, form $U_t$. Choose $(\varsigma_t,\gamma_t,N_t)$ satisfying the conclusion of \cref{prop:myula-accuracy} with accuracy $\tilde\alpha_t$. Initialize the chain at the fixed point $\theta_0\in\Theta$, run $N_t$ transitions of \eqref{eq:myula-step}, and project the final state onto $\Theta$. Denote the law of that projected state by $Q_t$. A warm start from the preceding output may be used in practice, but is not needed for the stated guarantee.
\end{onlinealgorithm}

\begin{remark}
\cref{prop:myula-accuracy} only requires the existence of a finite MYULA schedule attaining a prescribed total-variation accuracy. The explicit parameter choices and their quantitative complexity follow from the nonasymptotic bounds in Theorem 2, together with
Propositions 4 and 6 of \citet{brosse2017compact}. In particular, when the geometry of $\Theta$ is fixed and the smoothness parameters of $U_t$ grow at most polynomially in $t$, those bounds yield a
choice of $N_{t}$ that is polynomial in the relevant problem parameters, including $t$, $d$, and $\tilde\alpha$. We do not reproduce the resulting constants, since only the existence of a computable schedule is needed for the regret analysis below.
\end{remark}

\begin{theorem}[Fast regret of online MYULA]\label{thm:proximal-fast}
Under the same assumptions of \cref{thm:linear-statistical}, the approximate posterior computed by \cref{alg:proximal} satisfies
    \begin{equation}\label{eq:proximal-tracking}
    \E \sbr{\W^2(Q_t,\Pi_t)} \lesssim t^{-1},
    \end{equation}
and thus
    \begin{equation}\label{eq:proximal-fast-regret}
    \fR\del[1]{(Q_t)_{t=0}^{T-1}}\lesssim \log T+1.
    \end{equation}
\end{theorem}

\section{Truncated online posteriors for exponential family sequences}\label{sec:sequence}

We now turn to a different computational regime. In \cref{sec:constrained-langevin}, the model was finite-dimensional and the computational error was sampling error. Here the situation is reversed. Each coordinate posterior is available exactly in closed conjugate form, and the obstruction is representation: the posterior is a product of the coordinate posteriors, and its  worst-case memory is $O(T)$. The approximation we consider here updates only the first $m$ coordinates, thereby reducing the linear worst-case memory to a sublinear one.

\subsection{Infinite-dimensional exponential family  sequence models}
\label{sec:sequence-model}

For $\theta\in \R$, let $\bP_\theta$ be an one-dimensional regular canonical exponential family distribution, which admits a density with respect to a  dominating measure $\nu$ such that
    \begin{equation}\label{eq:nef}
    \sp_\theta(y)=\exp(y\theta-A(\theta)).
\end{equation}
Let $X_t$ be an i.i.d. positive integer-valued random variable such that $\Pr(X_t=j)=p_j$ for $j\in \bN$ with $\sum_{j=1}^\infty p_j=1$. Let $\theta^\star=(\theta^\star_j)_{j\ge1}$ be a true infinite-dimensional parameter. Conditional on $X_t=j$, we assume that $Y_t$ follows a distribution $\bP_{\theta^\star_j}$. We impose the following smoothness conditions on the true data-generating process.

\begin{assumption}[Sobolev sequence model]
\label[assumption]{ass:sequence_model}
The probability vector $(p_j)_{j\ge 1}$ satisfies
    \begin{equation}\label{eq:p-decay}
    p_j\asymp j^{-\fa}
    \end{equation}
for some $\fa>1$. The true parameter $\theta^\star$ belongs to the Sobolev ellipsoid
    \begin{equation}\label{eq:sobolev-ball}
    \ThetaSob
    =\cbr{(\theta_j)_{j\ge1}:\ |\theta_j|< B,\quad\sum_{j\ge1}j^{2\fs}\theta_j^2\le R^2}.
    \end{equation}
for some  $\fs>0$, $B>0$ and $R>0$. The log partition function $A$ is defined on the real line  and satisfies
    \begin{equation}\label{eq:A-regularity}
    0<\underline I\le A''(\theta)\le\overline I<\infty
    \end{equation}
for any $\theta\in[-B,B]$.
\end{assumption}

We assume that we correctly specify the exponential family sequence model, which means that we consider a negative log-density loss given by
    \begin{align*}
    \ell_t(\theta;(j,y))=-\log \sp_{\theta_j}(y)=A(\theta_{j})-y\theta_{j}
    \end{align*}
for $\theta=(\theta_j)_{j\ge 1}$, and the inverse temperature of $\eta=1$. Then the expected regret is represented as an expected KL risk. For a $\cG_{t-1}$-measurable distribution $Q$ of $\theta$, let $\hbP_Q(\cdot|j)$ be the associated predictive distribution given $X_t=j$ with density
    \begin{align*}
        \hsp_Q( y|j)
        =\int  \sp_{\theta_j}(y)Q(\d\theta)
        =\int \e^{-\ell_t(\theta;(j,y))}Q(\d\theta).
    \end{align*}
Then the conditional expectation of the excess predictive risk is
    \begin{align*}
        \cM_t(Q)-\cM_t(\delta_{\theta^\star})
        %&=\E\sbr{-\log \int \e^{-\ell_t(\theta;(X_t,Y_t))}Q_{t-1}(\d\theta)-\ell_t(\theta^\star;(X_t,Y_t))|\cG_{t-1}}\\
         &=\E\sbr{\log \frac{\e^{-\ell_t(\theta^\star;(X_t,Y_t))}}{\int \e^{-\ell_t(\theta;(X_t,Y_t))}Q(\d\theta)}|\cG_{t-1}}\\
        &=\sum_{j\ge1}p_j\int \log\frac{\sp_{\theta^\star_j}(y)}{   \hsp_Q( y|j)} \bP_{\theta^\star_j}(\d y)\\
        &=\sum_{j\ge1}p_j\KL\del[1]{\bP_{\theta^\star_j}\|\hbP_Q(\cdot|j) },
    \end{align*}
which is the expected KL divergence between the predictive distribution and the true conditional distribution.

We consider the prediction metric given by
    \begin{equation}\label{eq:dp-metric}
    \dist_p^2(\theta,\theta')
    =\sum_{j\ge1} p_j(\theta_j-\theta_j')^2,
    \end{equation}
and $\Wp$ denotes the associated Wasserstein distance. The next theorem shows that \cref{ass:predictive-stability} is satisfied for this exponential family sequence model

\begin{theorem}
\label{thm:sequence-stability}
Under \cref{ass:sequence_model}, there is an absolute constant $L_{\cM}>0$ such that
    \begin{align*}
    \cM_t(Q)-\cM_t(P)
    \le L_{\cM}\cbr{  \Wp(P,\delta_{\theta^\star})\Wp(P,Q)+\Wp^2(P,Q) }
\end{align*}
almost surely for every $t\in\bN$ and for any two product measures $P=\bigotimes_{j\ge 1}P_j$ and $Q=\bigotimes_{j\ge 1}Q_j$ on $\prod_{j\ge 1}[-B,B]$.
\end{theorem}

\subsection{Conjugate prior and posterior}

We consider a Diaconis--Ylvisaker-form conjugate prior \citep{diaconis1979conjugate} but here we restrict it to the fixed common interval $[-B,B]$.  Note that because this indicator does not depend on the data, the coordinatewise posterior update remains in the same truncated family. Specifically, we consider the prior distribution  $\Pi_0=\bigotimes_{j\ge1}\Pi_{0,j}$, where each coordinate prior is of the form
    \begin{equation}\label{eq:dy-prior}
    \Pi_{0,j}(\d\theta)\propto 
    \exp(\kappa_jA'(0)\theta-\kappa_jA(\theta))\ind_{[-B,B]}(\theta)\d\theta
    \end{equation}
for $\kappa_j\ge1$. We assume the following oracle choice of the hyperparameter $\kappa_j$ of the prior.

\begin{assumption}[Diaconis--Ylvisaker conjugate prior]
\label[assumption]{ass:dy_prior}
For every $j\in\bN$,  $\kappa_j\asymp j^{1+2\fs}$.
\end{assumption}

Due to conjugacy, the posterior distribution is given by $\Pi_t=\bigotimes_{j\ge1}\Pi_{t,j}$ with 
    \begin{equation}\label{eq:sequence-posterior}
    \Pi_{t,j}(\d\theta)
    \propto\exp\del{\{\kappa_jA'(0)+S_{t,j}\}\theta- \{\kappa_j+N_{t,j}\}A(\theta)}
    \ind_{[-B,B]}(\theta)\d\theta.
\end{equation}
where we let
     \begin{align*}
     N_{t,j}=\sum_{s=1}^t\ind\{X_s=j\},
    \qquad
     S_{t,j}=\sum_{s=1}^tY_s\ind\{X_s=j\}.
    \end{align*}

\begin{example}[Beta--Bernoulli sequence model]\label{cor:beta-bernoulli}
Let $Y_t|X_t=j\sim\mathrm{Bernoulli}(q_j)$ with $q_j=\sigma(\theta_j)$ for $\theta_j\in[-B,B]$, where $\sigma(\theta)=(1+\e^{-\theta})^{-1}$ denotes the logistic function. Then $A(\theta)=\log(1+\e^\theta)$ and $A''(\theta)=\sigma(\theta)\{1-\sigma(\theta)\}$ is bounded above and away from zero on $[-B,B]$. So \eqref{eq:A-regularity} is satisfied. Under the change of variable $q=\{1+\e^{-\theta}\}^{-1}$, the prior \eqref{eq:dy-prior} is a $\texttt{Beta}(\kappa_j/2,\kappa_j/2)$ law truncated to $[\sigma(-B),\sigma(B)]$, and the posterior is the corresponding truncated Beta update.
\end{example}

To ease the notation, we set
    \begin{equation}\label{eq:r-def}
    \zeta:=\fa+2\fs+1
     \end{equation}
throughout this section.

\begin{theorem}[Statistical guarantees for the posterior]
\label{thm:sequence-statistical}
Under \cref{ass:sequence_model,ass:dy_prior}, the exact posterior satisfies
    \begin{equation}\label{eq:sequence-contraction}
    \E\sbr{\Wp^2(\Pi_t,\delta_{\theta^\star})}
    \lesssim  t^{-(\zeta-1)/\zeta}
    \end{equation}
and
    \begin{equation}\label{eq:sequence-benchmark}
    \fR\del[1]{(\Pi_t)_{t=0}^{T-1}}
    \lesssim T^{1/\zeta}.
    \end{equation}
\end{theorem}

\subsection{Truncated posterior}\label{sec:spectral-algorithm}

Implementing the exact posterior requires storing sufficient statistics and computing posterior distributions for every distinct coordinate observed so far. Under the assumed regime $p_j\asymp j^{-\fa}$ the expected number of such coordinates is of order $T^{1/\fa}$. To improve computational efficiency, here we  propose an online approximate posterior that updates only the first $m\ll T^{1/\fa}$  coordinates (in our theoretical analysis, the oracle truncation level is of order $T^{1/(\fa+2\fs+1)}$), and preserves the original prior on the unresolved tail, which we call \textit{truncated posterior}. 

\begin{onlinealgorithm}[Online truncated posterior]
\label[onlinealgorithm]{alg:spectral}
Fix a truncation level $m\in \bN$. Store $(N_{t,j},S_{t,j})$ for $j\le m$. Before observing $Y_t$, after seeing $X_t=j$, output the truncated posterior
    \begin{equation}\label{eq:spectral-Q}
    Q_{t-1}^{(m)}
    =\bigotimes_{j\le m}\Pi_{t-1,j}\otimes \bigotimes_{j>m}\Pi_{0,j}
    \end{equation}
After observing $Y_t$, update $(N_{t,j},S_{t,j})$ only if $j\le m$.
\end{onlinealgorithm}

\cref{alg:spectral} uses $O(m)$ memory and $O(1)$ state-update work per observation, aside from evaluation of a one-dimensional conjugate normalizer ratio.

\begin{theorem}[Truncation error]\label{prop:spectral-error}
Under the same assumptions of \cref{thm:sequence-statistical}, the truncated posterior computed by  \cref{alg:spectral} satisfies
    \begin{equation}\label{eq:spectral-error}
    \E\sbr{\Wp^2(Q_t^{(m)},\Pi_t)}
    \lesssim m^{1-\zeta}+t^{-(\zeta-1)/\zeta}.
\end{equation}
\end{theorem}

\begin{theorem}[Fast regret of online truncated posterior]\label{thm:spectral-fast}
Under the same assumptions of \cref{thm:sequence-statistical}, the truncated posterior computed by  \cref{alg:spectral} satisfies
    \begin{equation}\label{eq:spectral-regret-general}
    \fR\del[1]{(Q_t^{(m)})_{t=0}^{T-1}}
    \lesssim   T^{1/\zeta}+T^{(\zeta+1)/(2\zeta)}m^{(1-\zeta)/2}+Tm^{1-\zeta}. 
    \end{equation}
In particular, when $m\gtrsim T^{1/\zeta}$,
    \begin{equation}\label{eq:spectral-fast-rate}
    \fR\del[1]{(Q_t^{(m)})_{t=0}^{T-1}}
    \lesssim T^{1/\zeta}.
    \end{equation}
\end{theorem}

It is worth emphasizing that the truncation here is computational rather than statistical. The last two rate terms in \eqref{eq:spectral-regret-general} monotonically decrease as the number of active coordinates $m$ grows. Whenever the order of $m$ is larger than  $T^{1/\zeta}$, the resulting regret is minimax optimal (see the next theorem). However, large $m$ requires more memory. The choice $m\asymp T^{1/\zeta}$ balances the statistical rate and computational burden, reducing expected memory from $T^{1/\fa}$ to $T^{1/\zeta}=T^{1/(\fa+2\fs+1)}$.

\begin{theorem}[Minimax lower bound]
\label{thm:sequence_minimax}
If $\widehat p_t(\cdot|X_t,\cG_{t-1})$ ranges over all sequential predictive densities that may use the revealed context $X_t$ and past history $\cG_{t-1}$, then
    \begin{align*}
    \inf_{(\widehat p_t)}\sup_{\theta^\star\in\ThetaSob}
    \E\sbr{ \sum_{t=1}^T\log\frac{p_{\star,X_t}(Y_t)}{\widehat p_t(Y_t| X_t,\cG_{t-1})}}
    \gtrsim T^{1/\zeta}.
    \end{align*}
\end{theorem}

\section{Online sparse variational Gaussian processes}
\label{sec:gp}

The preceding sequence example reduces representation cost by coordinate truncation. We now consider a different posterior approximation approach based on a variational Bayes algorithm for online nonparametric regression.

We introduce additional notation used in  this section. We write $\inn{f}{f'}:=\inn{f}{f'}_{L^2(\bP_X)}:=\int f(x)f'(x)\d \bP_X(x)$ denote the inner product on $L^2(\bP_X)$, and $\|f\|:=\|f\|_{L^2(\bP_X)}:=\sqrt{\inn{f}{f}}$ denote the associated norm. Moreover, $\W$ and $\opnorm{\cdot}$ denote the Wasserstein distance and the operator norm, respectively, both induced by the norm $\|\cdot\|$.

\subsection{Gaussian process regression}\label{sec:gp-model}

Let $\bP_X$ be a probability measure on a measurable input space $\cX$. At round $t$, we observe $X_t\stackrel{\mathrm{iid}}\sim \bP_X$ and then observe
    \begin{equation}\label{eq:gp-model}
    Y_t=f^\star(X_t)+\xi_t,
    \qquad  \xi_t\stackrel{\mathrm{iid}}\sim\N(0,\sigma_\xi^2),
    \end{equation}
where the noise $\xi_t$ is independent of $X_t$. To simplify the analysis, we assume the noise variance $\sigma_\xi^2$ is known to us. We use a negative log-density loss given by
    \begin{equation}\label{eq:gp-loss}
    \ell_t(f;(x,y))=\frac{1}{2\sigma_\xi^2}\{y-f(x)\}^2+\frac{1}{2}\log(2\pi\sigma_\xi^2)
    \end{equation}
and the inverse temperature of $\eta=1$. 

We impose a centered Gaussian process (GP) prior $\Pi_0$ on $f$ determined by its covariance kernel $\cK_0:\cX\times \cX\mapsto \R$, which has a Mercer decomposition
    \begin{align*}
    \sK_0(x,x')=\sum_{j\ge1}\lambda_j\psi_j(x)\psi_j(x'),
    \end{align*}
where $(\psi_j)_{j\ge1}$ is an orthonormal basis of $L_2(\bP_X)$. This is equivalent to considering a random series prior  $f=\sum_{j\ge1}\sqrt{\lambda_j}Z_j\psi_j$ with $Z_j\stackrel{\mathrm{iid}}\sim\N(0,1).$ The exact posterior is also a GP. Furthermore, the sparse GP posterior obtained using variational Bayes with inducing variables, which is our primary inference approach to be examined in the next subsection, is itself a GP.
%\begin{align*}
%\widehat f_t(x)&=\sK_0(x, X_{1:t}) (\sigma_\xi^2 I + \sK_0(X_{1:t},X_{1:t}))^{-1}Y_{1:t},\\
%\sK_t(x,x')&=\sK_0(x,x')-\sK_0(x, X_{1:t})(\sigma_\xi^2 I + \sK_0(X_{1:t},X_{1:t}))^{-1} \sK_0( X_{1:t},x').
%\end{align*}

For a GP $Q$, we let $m_Q(x)$ and $v_Q(x)$ denote the mean and variance of $f(x)$ for $f\sim Q$, respectively. The associated predictive distribution $\hbP_Q(\cdot|x)$ given $x\in\cX$ has a density
    \begin{align*}
     \hsp_{Q}(y|x) =\int \e^{-\ell_t(f;(x,y))} Q(\d f),
    \end{align*}
which is equal to the density of $\N( m_Q(x),\sigma_\xi^2+v_Q(x))$. When $Q$ is $\cG_{t-1}$-measurable, the excess conditional predictive risk is
    \begin{align*}
    \cM_t(Q)-\cM_t(\delta_{f^\star})
     &=\int \KL\del[1]{\N(f^\star(x),\sigma_\xi^2)\|\hbP_Q(\cdot|x)} \bP_X(\d x),
    \end{align*}
which is the expected KL divergence between the predictive distribution and the true conditional distribution. For this Gaussian regression setup, we cannot apply \cref{prop:primitive-stability} as the responses are unbounded. The next theorem overcomes this issue, which establishes second-order predictive stability with an additional remainder term.

\begin{theorem}[Predictive stability for GPs]
\label{thm:gp-predictive-comparison}
Let $P$ and $Q$ be Gaussian process laws on $L^2(\bP_X)$. Then there exists an absolute constant $L_{\cM}>0$ such that
    \begin{equation}\label{eq:gp-predictive-comparison}
    \cM_t(Q)-\cM_t(P)
    \le  L_{\cM}\cbr{ \W(P,\delta_{f^\star})\W(P,Q)+\W(P,Q)^2+\overline v(P,Q)\W(P,\delta_{f^\star})^2 }
\end{equation}
with $\overline v(P,Q):=\sup_{x\in\cX}\{v_P(x)+v_Q(x)\}$.
\end{theorem}

We assume that the basis functions are bounded and the true regression function belongs to a Sobolev space.

\begin{assumption}[Sobolev regression function]
\label[assumption]{ass:gp-true}
 There exist absolute constants $B_\psi>0$,  $\fs>d$, and $R>0$ such that  the orthonormal basis $(\psi_j)_{j\ge1}$ of $\mathcal L^2(\bP_X)$ satisfies
    \begin{equation}\label{eq:gp-basis}
    \sup_{j\ge1}\sup_{x\in\cX}|\psi_j(x)|
    \le B_\psi<\infty,
    \end{equation}
and the true regression function $f^\star=\sum_{j\ge1}f_j^\star\psi_j$ satisfies
    \begin{equation}\label{eq:gp-truth}
    \sum_{j\ge1}j^{2\fs/d}(f_j^\star)^2\le R^2.
    \end{equation}
\end{assumption}

We further assume that the scale of the GP prior is suitably chosen according to the Sobolev smoothness $\fs$.

\begin{assumption}[GP prior]
\label[assumption]{ass:gp-eigenvalues}
For every $j\in\bN$, $\lambda_j\asymp j^{-1-2\fs/d}.$
\end{assumption}

Under the assumptions we have made, we derive the fast regret bounds for the exact GP posterior.

\begin{theorem}[Statistical guarantees for the exact GP posterior]
\label{thm:gp-statistical}
Under \cref{ass:gp-true,ass:gp-eigenvalues}, the exact GP posterior satisfies
    \begin{equation}\label{eq:gp-contraction}
    \E\sbr{\W^2(\Pi_t,\delta_{f^\star})}
    \lesssim t^{-2\fs/(d+2\fs)}
    \end{equation}
and 
    \begin{equation}\label{eq:gp-exact-regret}
    \fR\del[1]{(\Pi_t)_{t=0}^{T-1}}
    \lesssim T^{d/(d+2\fs)}.
    \end{equation}
\end{theorem}

\subsection{Sparse variational Gaussian processes}\label{sec:gp-algorithm}

In practice, computing the GP posterior exactly becomes impractical for large $t$, because the computational cost grows cubically in $t$. To overcome this issue, we propose to use a sparse variational GP approach of \citet{titsias2009variational}, for online prediction. First, we choose the $J$-many inducing variables that are fixed for all $t\le T$. Define inducing variables as
    \begin{equation}\label{eq:gp-inducing}
    u_j:=\inn{f}{\psi_j},
    \qquad 1\le j\le J,
    \end{equation}
so that $u_{1:J}=(u_1,\ldots,u_J)^\top\sim\N(0,\Lambda_J)$ with
$\Lambda_J:=\operatorname{diag}(\lambda_1,\ldots,\lambda_J)$. Following \citet{titsias2009variational}, we consider a variational family $\mathcal Q_J$ which is the set of laws of the form
    \begin{equation}\label{eq:gp-var-family}
    Q(\d f,\d u)=\Pi_0(\d f\mid u)Q_u(\d u),
    \end{equation}
where $Q_u$ is a Gaussian distribution on $\R^J$. The sparse variational GP posterior is defined as
    \begin{equation}\label{eq:gp-variational-posterior}
    Q_t^{(J)}\in\mathop{\mathrm{argmin}}_{Q\in\mathcal Q_J}\KL(Q\|\Pi_t).
    \end{equation}
The minimizer  $Q_t^{(J)}$ in \eqref{eq:gp-variational-posterior} is the law of GP with mean and covariance functions given by (cf. see \citet{titsias2009variational, nieman2023uncertainty})
    \begin{align}
    \widehat f_t^{(J)}(x)&=\sigma_\xi^{-2}\psi_{1:J}(x)^\top V_{t,J}\Psi_{t,J}^\top Y_{1:t},\label{eq:gp-q-mean}\\
    \sK_t^{(J)}(x,x')
    &=\sK_0(x,x')-\psi_{1:J}(x)^\top\del{\Lambda_J-V_{t,J}}\psi_{1:J},(x')\label{eq:gp-q-covariance}
\end{align}
where we write $Y_{1:t}:=(Y_1,\dots, Y_t)^\top$, $\psi_{1:J}(x):=(\psi_1(x),\ldots,\psi_J(x))^\top$ and
     \begin{align*}
    V_{t,J}:=(\Lambda_J^{-1}+\sigma_\xi^{-2}\Psi_{t,J}^\top\Psi_{t,J})^{-1},
    \quad \Psi_{t,J}:=(\psi_j(X_i))_{i\le t, j\le J}.
    \end{align*}
Let $\widetilde\psi_{t+1}:=\psi_{1:J}(X_{t+1})$. Then the predictive distribution for $Y_{t+1}$ is given by
     \begin{align} \label{eq:gp-predictive-dist}
     \hbP_{Q_t^{(J)}}(\cdot|X_{t+1})
     =\N\del{\sigma_\xi^{-2}\widetilde\psi_{t+1}^\top V_{t,J}\Psi_{t,J}^\top Y_{1:t}, 
     \sigma_\xi^2 +\sK(X_{t+1},X_{t+1})-\widetilde\psi_{t+1}^\top\del{\Lambda_J-V_{t,J}}\widetilde\psi_{t+1}}.
     \end{align}

Since $\Psi_{t+1,J}^\top\Psi_{t+1,J}=\Psi_{t,J}^\top\Psi_{t,J}+\widetilde\psi_{t+1}\widetilde\psi_{t+1}^\top$, we can update $V_{t+1,J}$ from $V_{t,J}$ using Sherman–Morrison formula, as
    \begin{align} \label{eq:gp-vt-sherman}
    V_{t+1,J }= V_{t,J}
    -\frac{V_{t,J}\widetilde\psi_{t+1}\widetilde\psi_{t+1}^\top V_{t,J}}{\sigma_\xi^2+\widetilde\psi_{t+1}^\top V_{t,J} \widetilde\psi_{t+1}},
    \end{align}
which requires $O(J^2)$ computation. Based on this elementary observation, we propose an online algorithm that produces the sparse GP posterior \eqref{eq:gp-variational-posterior} exactly with $O(J^2)$ computational cost at each round. 

\begin{onlinealgorithm}[Online population-spectral sparse variational GP]
\label[onlinealgorithm]{alg:gp}
Fix $J$. For $t=0,1,\ldots,T$ we do:
\begin{enumerate}[label=(\roman*)]
\item after observing $X_{t+1}$, predict $Y_{t+1}$ using the predictive distribution $\hbP_{Q_t^{(J)}}(\cdot|X_{t+1})$ in \eqref{eq:gp-predictive-dist}.
\item update $V_{t+1,J}$ from $V_{t,J}$ according to \eqref{eq:gp-vt-sherman}.
\item after observing $Y_{t+1}$, compute $\Psi_{t+1,J}^\top Y_{1:(t+1)}=\Psi_{t,J}^\top Y_{1:t} + \widetilde\psi_{t+1} Y_{t+1}$.
\end{enumerate}
\end{onlinealgorithm}

\begin{theorem}[Approximation accuracy of sparse variational GP]
\label{thm:sparse-gp-accuracy}
Under the same assumption of \cref{thm:gp-statistical}, the sparse variational GP computed by \cref{alg:gp} with $J \gtrsim T^{d/(d+2\fs)}$  satisfies
    \begin{equation}\label{eq:gp-tracking}
    \E\sbr[1]{\W^2(Q_t^{(J)},\Pi_t)}
    \lesssim t^{-2\fs/(d+2\fs)}.
    \end{equation}
\end{theorem}

A notable feature of \cref{thm:sparse-gp-accuracy} is that the KL divergence between $Q_t^{(J)}$ and $\Pi_t$ itself need not vanish. Indeed, the available KL bound may grow with $t$, which is common in many applications \citep[e.g.,][]{zhang2020convergence,ohn2024adaptive,nieman2022contraction}. What matters for online prediction is that this information discrepancy is converted into Wasserstein displacement through the shrinking operator norm of the exact posterior covariance:
    \begin{align*}
        \W^2(Q_t^{(J)},\Pi_t) \lesssim \|\Sigma_t\|_{\mathrm{op}} \mathrm{KL}(Q_t^{(J)}\|\Pi_t).
    \end{align*}
with $\Sigma_t$ denoting the covariance operator of the exact GP posterior. For the proof, see \cref{lem:gp-transport} in the appendix. Thus the same variational discrepancy becomes less consequential as the exact posterior becomes geometrically more concentrated.

The exact and sparse posterior predictive variances are uniformly bounded; see \cref{lem:gp-pointvar} in the appendix. Hence \cref{thm:gp-predictive-comparison} yields the next result.

\begin{theorem}[Fast regret of online sparse variational GP]
\label{thm:gp-fast}
Under the same assumption of \cref{thm:gp-statistical}, the sparse variational GP computed by \cref{alg:gp} with $J \gtrsim T^{d/(d+2\fs)}$  satisfies
    \begin{equation}\label{eq:gp-fast-regret}
    \fR\del[1]{(Q_t^{(J)})_{t=0}^{T-1}}
    \lesssim T^{d/(d+2\fs)}.
    \end{equation}
\end{theorem}

With the oracle number of inducing variables $J\asymp T^{d/(d+2\fs)}$, the total computational cost through horizon $T$ of the sparse variational GPs $(Q_t^{(J)})_{t=0}^{T-1}$ is  
    \begin{equation}\label{eq:gp-complexity}
    O(TJ^2) =O\del[1]{T^{(3d+2\fs)/(d+2\fs)}}.
\end{equation}
By comparison, the exact GP posterior has  $O(t^2)$ update cost at round $t$, and thus total $O(T^3)$ cost through horizon $T$.

\section{Conclusion}\label{sec:discussion}

This paper develops a comparison principle for Bayesian online prediction with approximate posteriors. The main theorem separates the expected regret of a computed predictor into an exact Gibbs benchmark and an approximation penalty governed by two scales: the contraction radius $\varepsilon_t$ of the exact posterior and the tracking error $\alpha_t$ of the computational approximation. The penalty $\sum_{t=1}^{T-1}(\varepsilon_t\alpha_t+\alpha_t^2)$ makes precise how statistical concentration reduces sensitivity to computation. Posterior approximation need not be uniformly negligible; its accuracy only needs to improve at a rate compatible with contraction.

The analysis suggests a design principle for computational Bayesian prediction: first identify the contraction scale and predictive geometry of the exact posterior, then allocate only enough computational accuracy that the induced predictive perturbations remainbelow the statistical resolution of exact Gibbs posterior. This turns posterior approximation from an external numerical concern into an explicit component of sequential regret analysis. 

Several limitations mark useful directions for further work. The finite-dimensional result assumes compact support, smooth loss, and strong convexity supplied by regularization. Its cold-start mixing schedule is conservative, and warm starts or stochastic-gradient samplers may reduce its computational cost. The sequence result assumes known smoothness, a horizon-dependent rank, coordinatewise conjugacy, and tractable one-dimensional predictive normalizers. The sparse-GP result assumes known population spectral features, fixed kernel hyperparameters and inducing variables, polynomial spectral decay, and a known horizon. Extending the same direct comparison to fixed point-inducing schemes, moving inducing sets, adaptive ranks, or online hyperparameter learning requires new residual-kernel and sequential tracking arguments.

\section*{Acknowledgments}

This work was supported by the National Research Foundation of Korea (NRF) funded by the Korea government (MSIT) (RS-2026-25486448) and INHA UNIVERSITY Research Grant.

\appendix
\section{Proofs for \cref{sec:regret}}\label{app:regret}

\subsection{Proof of \cref{thm:main-transfer}}

\begin{proof}
Note that, at round $t$, $Q_{t-1}$ and $\Pi_{t-1}$ are $\cG_{t-1}$-measurable and
\cref{ass:predictive-stability} holds almost surely conditional on $\cG_{t-1}$. By the identity \eqref{eq:regret-conditional-risk}, \cref{ass:predictive-stability} gives
    \begin{align*}
        \fR\del[1]{(Q_t)_{t=0}^{T-1}}-\fR\del[1]{(\Pi_t)_{t=0}^{T-1}}
        \le L_{\cM}\sum_{t=2}^T
        \E\sbr{\W(\Pi_{t-1},\delta_{\theta^\star})\W(Q_{t-1},\Pi_{t-1})
        +\W^2(Q_{t-1},\Pi_{t-1})}.
    \end{align*}
Therefore, by the Cauchy--Schwarz inequality, we have
    \begin{align*}
    \fR\del[1]{(Q_t)_{t=0}^{T-1}}-\fR\del[1]{(\Pi_t)_{t=0}^{T-1}}
    &\le L_{\cM}\sum_{t=2}^T\sbr{\{
    \E[\W^2(\Pi_{t-1},\delta_{\theta^\star}]\}^{1/2}
    \{\E[\W^2(Q_{t-1},\Pi_{t-1})]\}^{1/2}
    +\E[\W^2(Q_{t-1},\Pi_{t-1})]}\\
    &=L_{\cM}\sum_{t=1}^{T-1}
    (\varepsilon_t\alpha_t+\alpha_t^2),
\end{align*}
which gives the desired result.
\end{proof}

\subsection{Proof of \cref{prop:gibbs-benchmark}}

\begin{proof}
By the Donsker-Varadhan variational formula, we have
    \begin{align*}
    -\log \int \e^{-F(\theta)} \Pi_0(\d\theta)
    \le \int F(\theta) \d\rho(\theta)+\KL(\rho\|\Pi_0)
    \end{align*}
for every $\rho\ll\Pi_0$.We then apply this inequality with
    \begin{align*}
    F(\theta)=\eta\sum_{t=1}^T\{\ell_t(\theta;(X_t, Y_t))-\ell_t(\theta^\star;(X_t, Y_t))\}.
    \end{align*}
\eqref{eq:telescoping} identifies $-\eta^{-1}\log\int \e^{-F(\theta)} \Pi_0(\d\theta)$ with $\regret((\Pi_t)_{t=0}^{T-1})$. Taking expectations and using Fubini under the stated integrability assumptions yields
    \begin{align*}
    \fR\del[1]{(\Pi_t)_{t=0}^{T-1}}
    \le \sum_{t=1}^T
    \int \E\{\ell_t(\theta;(X_t, Y_t))-\ell_t(\theta^\star;(X_t, Y_t))\}\rho(\d\theta)
    +\frac1\eta\KL(\rho\|\Pi_0).
\end{align*}
Taking the infimum over deterministic $\rho$ proves the proposition. 
\end{proof}

\subsection{Proof of \cref{prop:primitive-stability}}

\begin{proof}
Fix a round $t$. Let $Z_t:=(X_t, Y_t)$ In this proof, all expectations are conditional on $\cG_{t-1}$, and we suppress this conditioning from the notation for notational brevity. The conditional stationarity condition \eqref{eq:conditional-stationarity} becomes $\E[\nabla\ell_t(\theta^\star;Z_t)]=0$, while \eqref{eq:primitive-envelope} supplies a deterministic upper bound $C_{\rm env}$ for every conditional envelope moment used below.

Let $\gamma$ be an optimal coupling of $P$ and $Q$, and write
    \begin{align*}
     \theta_s:=(1-s)\theta+s\theta',\qquad v=\theta'-\theta,
    \end{align*}
for $(\theta,\theta')\sim\gamma$. Let $Q_s$ be the law of $\theta_s$. For fixed $z:=(x,y)$, define
    \begin{align*}
    f_z(s)=-\frac1\eta\log\int \e^{-\eta\ell_t(\theta_s;z)}\gamma(\d\theta,\d\theta').
     \end{align*}
Note that
    \begin{align*}
      \cM_t(Q)=\E[f_{Z_t}(1)\mid \cG_{t-1}],
      \quad
      \cM_t(P)=\E[f_{Z_t}(0)\mid \cG_{t-1}],
 \end{align*}
%Let $D_\Theta:=\operatorname{diam}(\Theta)$. The first derivative of the integrand $s\mapsto \e^{-\eta\ell_t(\theta_s,z)}$ is bounded in absolute value by $\eta G_tD_\Theta \e^{-\eta\ell_t(\theta_s,z)}$, and after normalization the tilted first derivative is bounded by $G_tD_\Theta$. Since $G_t\le 1+G_t^2$ and \eqref{eq:primitive-envelope} bounds the conditional expectation of $\e^{\eta D_t}(1+G_t^2)$, differentiation under the conditional expectation is justified at first order.
Because $\nabla\ell_t(\cdot,z)$ is $H_t$-Lipschitz, $s\mapsto\ell_t(\theta_s,z)$ has a Lipschitz derivative and is twice differentiable for Lebesgue-a.e. $s$, with
    \begin{align*}
    \left|\frac{\d^2}{\d s^2}\ell_t(\theta_s;z)\right|
    \le H_t\|v\|^2.
    \end{align*}
Let $G_{s,z}$ be a probability measure on $\Theta\times \Theta$ defined as
    \begin{align*}
    G_{s,z}(\d\theta,\d\theta')
    =\frac{ \e^{-\eta\ell_t(\theta_s;z)}\gamma(\d\theta,\d\theta')}
    {\int \e^{-\eta\ell_t(u_s;z)}\gamma(\d u,\d u')}, 
    \quad u_s:=(1-s)u + su'.
    \end{align*}
Differentiation of the log integral gives
    \begin{align*}
    f_z'(s)
    =\frac{\int\frac{\d}{\d s}\ell_t(\theta_s;z)\e^{-\eta\ell_t(\theta_s;z)}\gamma(\d\theta,\d\theta')}
    {\int \e^{-\eta\ell_t(u_s,z)}\gamma(\d u,\d u')}
    =\E_{G_{s,z}}\sbr{\frac{\d}{\d s}\ell_t(\theta_s;z)}
    \end{align*}
where $\E_{G_{s,z}}$ is expectation under the probability measure $G_{s,z}$. Moreover, we have
    \begin{align*}
    f_z''(s)&=
    \E_{G_{s,z}}\sbr{\frac{\d^2}{\d s^2}\ell_t(\theta_s;z)}
    -\eta \E_{G_{s,z}}\sbr{\del[1]{\frac{\d}{\d s}\ell_t(\theta_s;z)}^2}
    +\eta\del{\E_{G_{s,z}}\sbr{\frac{\d}{\d s}\ell_t(\theta_s;z)}}^2\\
    &=   \E_{G_{s,z}}\sbr{\frac{\d^2}{\d s^2}\ell_t(\theta_s;z)}
    -\eta\text{Var}_{G_{s,z}}\sbr{\frac{\d}{\d s}\ell_t(\theta_s;z)}.
 \end{align*}
By assumption, we have
    \begin{align} \label{eq:tilted_density_bound}
    \frac{ \e^{-\eta\ell_t(\theta_s;z)}}{\int \e^{-\eta\ell_t(u_s;z)}\gamma(\d u,\d u')}
    =\frac{ 1}{\int \e^{\eta\{\ell_t(\theta_s;z)-\ell_t(u_s;z)\}}\gamma(\d u,\d u')}
    \le \frac{1}{\e^{-\eta D_t}}
    \end{align}
and therefore,
    \begin{align*}
    f_z''(s)
    &\le \E_{G_{s,z}}\sbr{\frac{\d^2}{\d s^2}\ell_t(\theta_s;z)}
    \le H_t\e^{\eta D_t}\W^2(P,Q).
    \end{align*}
Since $|f_z''(s)|\lesssim \e^{\eta D_t}( H_t+\eta G_t^2)D_\Theta^2$ and this is conditionally integrable by \eqref{eq:primitive-envelope}, Fubini and the fundamental theorem for absolutely continuous functions therefore yield
    \begin{equation}\label{eq:app-geodesic-smoothness}
    \cM_t(Q)-\cM_t(P)
    \le   \left.\frac{\d}{\d s}\cM_t(Q_s)\right|_{s=0}
    +\frac{L_1}{2}\W^2(P,Q),
\end{equation}
where $ L_1:=\E[H_t\e^{\eta D_t}\mid\cG_{t-1}]\le C_{\rm env}.$

It remains to control the directional derivative. Define
    \begin{align*}
    w_P(\theta,z)=\frac{\e^{-\eta\ell_t(\theta;z)}}{\int \e^{-\eta\ell_t(u,z)}P(\d u)},
    \end{align*}
and
    \begin{align*}
    b_{P,t}(\theta)
    =\E[w_P(\theta,Z_t)\nabla\ell_t(\theta;Z_t)\mid\cG_{t-1}].
    \end{align*}
Then we have
    \begin{equation}\label{eq:directional-first-variation}
    \left.\frac{\d}{\d s}\cM_t(Q_s)\right|_{s=0}
    =\int\langle b_{P,t}(\theta),\theta'-\theta\rangle\gamma(\d\theta,\d\theta').
    \end{equation}
Let $g_{\star,t}:=\nabla\ell_t(\theta^\star;Z_t)$ and $m_P:=\int\|u-\theta^\star\|P(\d u)$. The stationarity assumption gives $\E(g_{\star,t}\mid\cG_{t-1})=0$, so
    \begin{align*}
    b_{P,t}(\theta)
    =\E[w_P(\theta,Z_t)[\nabla\ell_t(\theta;Z_t)-g_{\star,t}]\mid\cG_{t-1}]
    +\E[\{w_P(\theta,Z_t)-1\}g_{\star,t}\mid\cG_{t-1}].
    \end{align*}
Using a similar argument to \eqref{eq:tilted_density_bound}, the first term is bounded by
$\E(\e^{\eta D_t}H_t\mid\cG_{t-1})\|\theta-\theta^\star\|$. For the second, the mean-value theorem applied to $u\mapsto \e^{-\eta\ell_t(u;z)}$ gives
    \begin{align*}
    |w_P(\theta,z)-1|
    & \le \abs{\frac{\int\e^{-\eta\ell_t(\theta;z)}- \e^{-\eta\ell_t(u;z)}P(\d u)}{\int \e^{-\eta\ell_t(u;z)}P(\d u)}}\\
    & \le \abs{\frac{\int \inn{\sup_{u'\in \Theta}-\eta\e^{-\eta\ell_t(u';z)}\grad \ell_t(u';z)}{\theta-u}  P(\d u)}{\int \e^{-\eta\ell_t(u;z)}P(\d u)}}\\
    & \le \abs{\frac{\int\eta G_t\|\theta-u\|  P(\d u)}{\int \e^{\sup_{u'\in \Theta}\ell_t(u';z)-\eta\ell_t(u;z)}P(\d u)}}\\
    &\le \eta \e^{\eta D_t}G_t \{\|\theta-\theta^\star\|+m_P\}.
    \end{align*}
Because $\|g_{\star,t}\|\le G_t$, by \eqref{eq:primitive-envelope},
    \begin{align*}
    \|b_{P,t}(\theta)\|
    \le C_{\rm env}\{\|\theta-\theta^\star\|+m_P\}.
\end{align*}
a.s.. By Jensen's inequality, we have $m_P\le \W(P,\delta_{\theta^\star})$, and hence
    \begin{align*}
    \|b_{P,t}\|_{L^2(P)}
    \le 4C_{\rm env}\W(P,\delta_{\theta^\star}).
    \end{align*}
Combining this with \eqref{eq:directional-first-variation} and using Cauchy--Schwarz under $\gamma$, we have
    \begin{align*}
    \cM_t(Q)-\cM_t(P)
    \le    4C_{\rm env} \W(P,\delta_{\theta^\star})\W(P,Q)
    +\frac{L_1}{2}\W^2(P,Q),
\end{align*}
which completes the proof.
\end{proof}

\section{Proofs for \cref{sec:langevin}}\label{app:langevin}

\subsection{Proof of \cref{thm:linear_stability}}

\begin{proof}
%Given $\cG_{t-1}$, the covariate $X_t$ is fixed.
It is easy to see that the loss $\theta\mapsto\ell_{\lambda,t}(\theta;Y_t)$ satisfies \eqref{eq:loss_grad_lip}, \eqref{eq:loss_grad_bound} and \eqref{eq:loss_lip} with
     \begin{align*}
     H_t&=L_\phi K_x^2+\lambda\\
     G_t&=B_1K_x+\lambda\operatorname{diam}(\Theta)\\
     D_t&=G_t\operatorname{diam}(\Theta),
\end{align*}
These constants are fixed in $t$, hence \eqref{eq:primitive-envelope} is met.
\cref{ass:stationarity} is exactly \eqref{eq:linear-mds}. Thus \cref{prop:primitive-stability} can be applied, which leads to \cref{ass:predictive-stability}.
\end{proof}

\subsection{Proof of \cref{thm:linear-statistical}}

We first record a self-contained moment bound that will be used in both finite- and infinite-dimensional applications.

\begin{lemma}[Second moment under a constrained strongly convex Gibbs law]
\label[lemma]{lem:strong-convex-moment}
Let $K\subset\R^d$ be a compact convex set with nonempty interior and let $\theta_0\in\operatorname{int}(K)$. Suppose $U$ is continuously differentiable on a neighborhood of $K$ and $\beta$-strongly convex on $K$. Then for a probability measure $\check\Pi$ on $\R^d$ such that
    \begin{align*}
    \check\Pi(\d\theta)=Z^{-1}\e^{-U(\theta)}\ind_K(\theta)\d\theta,
    \end{align*}
we have
    \begin{align}\label{eq:strong-convex-moment}
    \int\|\theta-\theta_0\|^2\check\Pi(\d\theta)
    \le \frac{2\|\nabla U(\theta_0)\|^2}{\beta^2}+\frac{2d}{\beta}.
    \end{align}
\end{lemma}

\begin{proof}
A compact convex body has Lipschitz boundary, so the divergence theorem applies. At surface-a.e. $x\in\partial K$, the outer normal $n(x)$ satisfies $\langle x-\theta_0,n(x)\rangle\ge0$ by convexity. Hence
    \begin{align*}
    0\le \int_{\partial K}\langle x-\theta_0,n(x)\rangle \e^{-U(x)}\d S(x)
    = \int_K\{d-\langle x-\theta_0,\nabla U(x)\rangle\}\e^{-U(x)}\d x.
    \end{align*}
This implies that 
    \begin{align*}
    \E_{\check\Pi}[\langle X-\theta_0,\nabla U(X)\rangle]\le d.
    \end{align*}
Strong convexity implies strong monotonicity of the gradient,
    \begin{align*}
    \langle X-\theta_0,\nabla U(X)-\nabla U(\theta_0)\rangle
    \ge \beta\|X-\theta_0\|^2.
    \end{align*}
Let $g_0:=\|\nabla U(\theta_0)\|$ and $s^2:=\E_{\check\Pi}\|X-\theta_0\|^2$. Taking expectations and using Cauchy--Schwarz gives $\beta s^2\le d+g_0s$. Solving this quadratic inequality yields
$s\le g_0/\beta+\sqrt{d/\beta}$, proving the desired result.
\end{proof}

\begin{proof}[Proof of \cref{thm:linear-statistical}]
Let
    \begin{align*}
    \Xi_s:=\nabla\ell_{\lambda,s}(\theta^\star;Y_s).
    \end{align*}
By \eqref{eq:linear-mds}, $(\Xi_s)$ is a martingale-difference sequence with respect to the pre-prediction filtration. \cref{ass:design,ass:linear-loss} imply
$\|\Xi_s\|\le G_0$ almost surely for some absolute constant $G_0>0$. 
Due to the $L^2$ regularization term $(\lambda/2)\|\theta\|^2$, the potential function is $\beta_t$-strongly convex on $\Theta$ with $\beta_t:=\eta\lambda t$. The gradient of  $\overline U_t$ at the target is
    \begin{align*}
    \nabla\overline U_t(\theta^\star)
    =\nabla V_0(\theta^\star)+\eta\sum_{s=1}^t\Xi_s.
    \end{align*}
The martingale differences are orthogonal in $L^2$: for $r<s$,
    \begin{align*}
    \E[\langle\Xi_r,\Xi_s\rangle]
    =\E\sbr{\left\langle\Xi_r,\E(\Xi_s\mid\cG_{s-1})\right\rangle}
    =0.
    \end{align*}
Consequently,
    \begin{align*}
    \E\left\|\sum_{s=1}^t\Xi_s\right\|^2
    =\sum_{s=1}^t\E[\|\Xi_s\|^2]
    \le G_0^2t,
    \end{align*}
and hence
    \begin{align*}
    \E\|\nabla\overline U_t(\theta^\star)\|^2\lesssim t.
    \end{align*}
By applying \cref{lem:strong-convex-moment} with $K=\Theta$ and $\theta_0=\theta^\star$, we have
    \begin{align*}
    \E [\W^2(\Pi_t,\delta_{\theta^\star})]
    &= \E\sbr{\int\|\theta-\theta^\star\|^2\Pi_t(\d\theta)}
    \lesssim \frac d t,
\end{align*}
which proves \eqref{eq:linear-contraction}.

We now verify \eqref{eq:linear-gibbs}. Fix $c>0$ sufficiently small and let $r_T=cT^{-1/2}$. Because $\theta^\star$ is interior, $B(\theta^\star,r_T)\subset\Theta$ for all sufficiently large $T$. Let $\rho_T$ be the uniform distribution on this ball. The prior density is bounded below by a positive constant on a fixed neighborhood of $\theta^\star$, and hence
    \begin{align*}
    \KL(\rho_T\|\Pi_0)
    \lesssim  -\log\operatorname{Vol}(B(0,r_T))+1
    \lesssim \log T+1.
\end{align*}
Define the conditional risk
    \begin{align*}
    R_{\lambda,t}(\theta)=\E[\ell_{\lambda,t}(\theta;Y_t)\mid\cG_{t-1}].
    \end{align*}
By assumption, this has $L_R$-Lipschitz gradient with $L_R:=L_\phi K_x^2+\lambda$. Moreover, by \eqref{eq:linear-mds}, $\nabla R_{\lambda,t}(\theta^\star)=0$. Therefore, we have
    \begin{align*}
    R_{\lambda,t}(\theta)-R_{\lambda,t}(\theta^\star)
    \le\frac{L_R}{2}\|\theta-\theta^\star\|^2
    \end{align*}
almost surely. Taking expectations and summing over $t$ yields
    \begin{align*}
    \sum_{t=1}^T \int\E[\ell_{\lambda,t}(\theta;Y_t)-\ell_{\lambda,t}(\theta^\star;Y_t)]
    \rho_T(\d\theta)
    \lesssim 1,
\end{align*}
where the finitely many remaining small values of $T$ are absorbed into the implicit constant. \cref{prop:gibbs-benchmark} proves the desired result.
\end{proof}

\subsection{Proof of \cref{prop:myula-accuracy}}

\begin{proof}
We apply Theorem 2 of \citet{brosse2017compact} with $K=\Theta$ and $f=U_t$ and so it suffices to check Assumptions H1 and H2 therein. \cref{ass:linear-loss} implies that $U_t$ is convex and continuously differentiable on $\R^d$ with $L_t$-Lipschitz gradient, where $L_t\lesssim (1+t)$. These verify H1. The geometry assumption in H2 is already assumed in our \cref{ass:linear-loss}.
\end{proof}

\subsection{Proof of \cref{thm:proximal-fast}}

\begin{proof}
Let $\widetilde Q_t$ be the law before the final projection in \cref{alg:proximal}. \cref{prop:myula-accuracy} gives
    \begin{align*}
    \|\widetilde Q_t-\Pi_t\|_{\mathrm{TV}}\le \tilde\alpha_t.
    \end{align*}
Because $\operatorname{proj}_\Theta$ is measurable and $\operatorname{proj}_\Theta\#\Pi_t=\Pi_t$, total variation contracts under pushforward:
    \begin{align*}
    \|Q_t-\Pi_t\|_{\mathrm{TV}
    }=    \|\operatorname{proj}_\Theta\#\widetilde Q_t-\operatorname{proj}_\Theta\#\Pi_t\|_{\mathrm{TV}}
    \le  \tilde\alpha_t.
 \end{align*}
Both $Q_t$ and $\Pi_t$ are supported on $\Theta$. Maximal coupling therefore yields
     \begin{align*}
      \W^2(Q_t,\Pi_t)\le
      \operatorname{diam}(\Theta)^2\|Q_t-\Pi_t\|_{\mathrm{TV}}
      \lesssim t^{-1}.
      \end{align*}
Taking expectations proves \eqref{eq:proximal-tracking}.

\cref{thm:linear-statistical} gives $\varepsilon_t\lesssim t^{-1/2}$. Consequently,
    \begin{align*}
    \sum_{t\ge1}\{\varepsilon_t\alpha_t+\alpha_t^2\}
    \lesssim \sum_{t\ge1}t^{-1}
    \lesssim \log T.
    \end{align*}
Applying \cref{thm:main-transfer} with \cref{thm:linear-statistical} completes the proof.
\end{proof}

\section{Proofs for \cref{sec:sequence}}\label{app:sequence}

\subsection{Supporting lemmas}

We first collect four elementary lemmas. 

\begin{lemma}[Wasserstein distance for product measures]
\label[lemma]{lem:product-w2}
Let $(p_j)_{j\ge1}$ satisfy $p_j>0$ for every $j$ and $\sum_jp_j=1$. Let $\mu=\bigotimes_{j\ge1}\mu_j$ and $\nu=\bigotimes_{j\ge1}\nu_j$ be probability measures on the sequence space equipped with the cost $\dist_p^2(x,y)=\sum_{j\ge1}p_j(x_j-y_j)^2$. If
$\sum_jp_j\W^2(\mu_j,\nu_j)<\infty$, then
    \begin{equation}\label{eq:product-w2}
    \Wp^2(\mu,\nu)=\sum_{j\ge1}p_j\W^2(\mu_j,\nu_j).
    \end{equation}
\end{lemma}

\begin{proof}
Let $\Gamma$ be any coupling of $\mu$ and $\nu$, and let $\Gamma_j$ be its $j$th coordinate marginal. By Tonelli,
    \begin{align*}
    \int\dist_p^2(x,y)\Gamma(\d x,\d y)
    =\sum_{j\ge1}p_j\int(x_j-y_j)^2\Gamma_j(\d x_j,\d y_j)
    \ge\sum_{j\ge1}p_j\W^2(\mu_j,\nu_j).
    \end{align*}
Taking the infimum over $\Gamma$ gives the lower bound. Conversely, for each $j$, choose an optimal quadratic-cost coupling $\gamma_j$ of $\mu_j$ and $\nu_j$; such a coupling exists on $\mathbb R$ whenever the displayed Wasserstein distance is finite. The countable product measure $\bigotimes_j\gamma_j$ is a coupling of $\mu$ and $\nu$, and Tonelli's theorem gives its total cost as exactly $\sum_jp_j\W^2(\mu_j,\nu_j)$. This proves the reverse inequality.
\end{proof}

\begin{lemma}[One-coordinate posterior risk]
\label[lemma]{lem:coordinate-risk}
Under \cref{ass:sequence_model}, there is an absolute constant $C>0$ such that
    \begin{equation}\label{eq:coordinate-risk}
    \E\sbr{\left.\int(\theta-\theta^\star_j)^2\Pi_{t,j}(\d\theta)\right| N_{t,j}}
    \le C \cbr{\frac1{\kappa_j+N_{t,j}}
    +\frac{\kappa_j^2(\theta^\star_j)^2}{(\kappa_j+N_{t,j})^2}}
    \end{equation}
for any $j\in \bN$ and $t\in \bN$.
\end{lemma}

\begin{proof}
The potential function of the posterior $\Pi_{t,j}$ on $[-B,B]$ is given by
    \begin{align*}
    U_{t,j}(\theta)
    =(\kappa_j+N_{t,j})A(\theta)-\{\kappa_j A'(0)+S_{t,j}\}\theta,
    \end{align*}
which is $\beta$-strongly convex with $\beta=(\kappa_j+N_{t,j})\underline I$. Since $|\theta^\star_j|< B$,  we can apply \cref{lem:strong-convex-moment}. Note that
    \begin{align*}
    U_{t,j}'(\theta^\star_j)
    =\kappa_j\{A'(\theta^\star_j)-A'(0)\}
    -\{S_{t,j}-N_{t,j}A'(\theta^\star_j)\}.
    \end{align*}
By the well-known properties of an exponential family distribution, conditional on $N_{t,j}$, the mean and variance of the sufficient statistics $S_{t,j}$ is $N_{t,j}A'(\theta^\star_j)$ and $N_{t,j}A''(\theta^\star_j)$, respectively. Since  $A''(\theta^\star_j)\le \overline I$ and
$|A'(\theta^\star_j))-A'(0)|\le\overline I|\theta^\star_j|$, we have
    \begin{align*}
    \E\sbr{|U_{t,j}'(\theta^\star_j)|^2|N_{t,j}}
    \lesssim N_{t,j}+\kappa_j^2(\theta^\star_j)^2.
    \end{align*}
Taking expectation in \eqref{eq:strong-convex-moment} therefore gives the desired bound.
\end{proof}

\begin{lemma}[Binomial resolvents]
\label[lemma]{lem:binomial-resolvents}
If $N\sim\mathrm{Binomial}(t,p)$ and $\kappa\ge1$, then
    \begin{equation}\label{eq:binomial-resolvents}
    \E\sbr[3]{\frac1{\kappa+N}}
    \le\frac C{\kappa+tp},
    \qquad
    \E\sbr[3]{\frac1{(\kappa+N)^2}}
    \le\frac C{(\kappa+tp)^2}
    \end{equation}
for some absolute constant $C>0$.
\end{lemma}

\begin{proof}
Let $\mu=tp$. If $\mu<2$, then $\kappa+\mu\le3\kappa$ because $\kappa\ge1$, and the claim follows from 
    \begin{align*}
    (\kappa+N)^{-q}\le\kappa^{-q}\le ((\kappa+\mu)/3)^{-q}
    \end{align*}
for $q\in\{1,2\}$. Assume $\mu\ge2$ and let $E=\{N\ge\mu/2\}$. On $E$,
    \begin{align*}
    (\kappa+N)^{-q}\le2^q(\kappa+\mu)^{-q}.
    \end{align*}
For the complement, for any $s>0$, Markov's inequality gives
    \begin{align*}
    \Pr(N\le\mu/2)
    &\le \e^{s\mu/2}\E [\e^{-sN}]\\
    &=\e^{s\mu/2}(1-p+p\e^{-s})^t\\
    &\le\exp\{\mu(s/2+\e^{-s}-1)\}.
\end{align*}
With $s=\log2$, this is at most $\e^{-c_0\mu}$, where $c_0:=(1-\log2)/2>0$. Thus
    \begin{align*}
    \E[(\kappa+N)^{-q}\mathbf1_{E^c}]
    \le\kappa^{-q}\e^{-c_0\mu}.
\end{align*}
Since $x^q\e^{-c_0x}$ is bounded on $[0,\infty)$ and $\kappa\ge1$,
$\e^{-c_0\mu}\lesssim\{\kappa/(\kappa+\mu)\}^q$. Combining the bounds on $E$ and $E^c$ proves \eqref{eq:binomial-resolvents}.
\end{proof}

\begin{lemma}[A localized comparator for a conjugate prior]
\label[lemma]{lem:coordinate-comparator}
Let $\Pi$ be a distribution with $\Pi_\kappa(\d\theta)\propto \exp(\kappa A'(0)\theta-\kappa A(\theta))\ind_{[-B,B]}\d\theta$ with  $\kappa\ge1$. Assume \eqref{eq:A-regularity}. For $|\theta^\star|<B$ and $n\in\bN$, there exists a probability measure $\rho$ on $[-B,B]$ such that
    \begin{align}
    \KL(\rho\|\Pi)
    &\le C\left\{1+\log\left(1+\frac n\kappa\right)+\kappa(\theta^\star)^2\right\},
    \label{eq:coordinate-kl}\\
    \int D_A(\theta,\theta^\star)\rho(\d\theta)
    &\le\frac C{\kappa+n},\label{eq:coordinate-risk-comparator}
    \end{align}
for some absolute constant $C>0$, where
$D_A(\theta,\theta'):=A(\theta)-A(\theta')-A'(\theta')(\theta-\theta')$.
\end{lemma}

\begin{proof}
Let $I_\star\subset[-B,B]$ be a small neighborhood of $\theta^\star$. Put
$v=c_0/(\kappa+n)$ with sufficiently small $c_0$, and let $\rho$ be the $\N(\theta^\star,v)$ law conditioned on $I_\star$. %The conditioning probability is bounded below uniformly because $b/\sqrt v\ge b/\sqrt{c_0}\ge1$.
Then
    \begin{align*}
    \int(\theta-\theta^\star)^2\rho(\d\theta)\lesssim v.
    \end{align*}
Moreover the conditional Gaussian density is bounded by $v^{-1/2}$, up to a multiplicative constant.
%with $C$ depending only on the uniform lower bound on the conditioning probability.
Therefore, the entropy of $\rho$ is bounded below as
    \begin{align*}
    h(\rho)&:=-\int \log(\rho(\d\theta)/\d\theta)\rho(\d\theta)\\
    &\ge -\log\|\rho(\d\theta)/\d\theta\|_\infty\\
    &\ge -\frac{1}{2}\log(v) + C_1'\\
    &\ge -\frac{1}{2}\log(\kappa+n) + C_2'
 \end{align*}
for some absolute constants $C_1'$ and $C_2'$. The prior can be written as
    \begin{align*}
    \Pi(\d\theta)=Z_0(\kappa)^{-1}
    \e^{-\kappa D_A(\theta,0)}\mathbf1_{[-B,B]}(\theta)\d\theta,
    \end{align*}
where $Z_0(\kappa)=\int_{-B}^B\e^{-\kappa D_A(u,0)}\d u$. The curvature bounds \eqref{eq:A-regularity} on $A$ imply
    \begin{align*}
    \frac{\underline I}{2}u^2\le D_A(u,0)\le\frac{\overline I}{2}u^2,
    \qquad |u|\le B.
    \end{align*}
Integrating the lower quadratic bound gives
$\log Z_0(\kappa)\le C_3'-\frac12\log\kappa$ for some absolute constants $C_3'$. Also,
    \begin{align*}
    \kappa\int D_A(\theta,0)\rho(\d\theta)
    \lesssim \kappa\{(\theta^\star)^2+v\}
    \lesssim \kappa(\theta^\star)^2+1.
\end{align*}
Therefore
    \begin{align*}
    \KL(\rho\|\Pi)
    &=-h(\rho)+\kappa\int D_A(\theta,0)\rho(\d\theta)+\log Z_0(\kappa)\\
    &\lesssim\left\{1+\log\left(1+\frac n\kappa\right)+\kappa(\theta^\star)^2\right\},
    \end{align*}
which proves \eqref{eq:coordinate-kl}. Finally,
$D_A(\theta,\theta^\star)\le(\overline I/2)(\theta-\theta^\star)^2$, so
    \begin{align*}
    \int D_A(\theta,\theta^\star)\rho(\d\theta)
    \lesssim v
    \lesssim\frac1 {\kappa+n},
    \end{align*}
proving \eqref{eq:coordinate-risk-comparator}.
\end{proof}

\subsection{Proof of \cref{thm:sequence-stability}}

\begin{proof}
For a $\cG_{t-1}$-measurable distribution $Q$, since the data are i.i.d., we have
   \begin{align*}
   \cM_t(Q)
   &=\E\sbr{-\log \int \e^{-\ell_t(\theta;(X_t,Y_t))}Q(\d\theta)|\cG_{t-1}}\\
   &=\sum_{j\ge1}p_j\E\sbr{-\log \int \e^{-\ell_t(\theta;(j,Y_t))}Q(\d\theta)\big|\cG_{t-1}, 
   X_t =j}\\
   &=\sum_{j\ge1}p_j\widetilde\cM_{j}(Q),
    \end{align*}
where we define the per-coordinate conditional predictive risk 
    \begin{align*}
         \widetilde\cM_{j}(Q):=\E\sbr{-\log \int \e^{-\ell_t(\theta;(j,Y_t))}Q(\d\theta)\big|\cG_{t-1}, X_t=j}.
    \end{align*}
Note that the gradient of $\ell_t(\theta;(j,y))$ with respect to the first argument is $A'(\theta_j)-y$ and its gradient-Lipschitz constant is at most $\overline I$. Put $M_A=\sup_{|u|\le B}|A'(u)|<\infty$. Then we can take the explicit envelopes
    \begin{align*}
    H(y)&=\overline I,\\
    G(y)&=M_A+|y|,\\
    D(y)&=2BM_A+2B|y|,
    \end{align*}
where we use the mean-value theorem giving $|A(\theta)-A(u)|\le M_A|\theta-u|$ and $|\theta-u|\le2B$ for the last line. The elementary bound $1+y^2\le \e^{|y|}$ implies
    \begin{align*}
    \e^{D(y)}\{H(y)+G^2(y)\}
    \lesssim \e^{(2B+1)|y|}.
    \end{align*}
Since the log-partition function is assumed to exist on the whole real line, the right-hand side of the above display is integrable, which verifies \eqref{eq:primitive-envelope}. Hence, by \cref{prop:primitive-stability}, there is an absolute constant $C'>0$ such that 
    \begin{align*}
    \widetilde\cM_{j}(Q_j)- \widetilde\cM_{j}(P_j)
    \le C'\{\W(P_j,\delta_{\theta^{\star}_{j}})\W(P_j,Q_j)+\W(P_j,Q_j)^2\},
    \end{align*}
for any $j\in\bN$. In view of \cref{lem:product-w2}, by multiplying by $p_j$, summing, and applying Cauchy--Schwarz, we have
    \begin{align*}
    \cM_t(Q)-\cM_t(P)
    &=\sum_{j\ge1}p_j\{\widetilde\cM_{j}(Q_j)-\widetilde\cM_{j}(P_j)\}\\
    &\lesssim \cbr{\Wp(P,\delta_{\theta^\star})\Wp(P,Q)+\Wp^2(P,Q)},
\end{align*}
which completes the proof.
\end{proof}

\subsection{Proof of \cref{thm:sequence-statistical}}

\begin{proof}
By \cref{lem:coordinate-risk}, followed by \cref{lem:binomial-resolvents}, there exists an absolute constant $C'_1>0$ such that
    \begin{align}
    \E\sbr{\int(\theta-\theta^\star_j)^2\Pi_{t,j}(\d\theta)}
    &\le C'_1\cbr{\frac1{\kappa_j+tp_j}+
    \frac{\kappa_j^2(\theta^\star_j)^2}{(\kappa_j+tp_j)^2} }\notag \\
    &= C'_1\cbr{  \frac{q_j}{1+tq_j}+\frac{p_j\theta_{\star,j}^2}{(1+tq_j)^2} } \label{eq:coordinate-after-binomial}
    \end{align}
with $q_j:=p_j/\kappa_j\asymp j^{-\zeta}$, for any $j\in\bN$ and $t\in \bN$. By \cref{lem:product-w2}, squared Wasserstein distances for these product measures add coordinatewise under the cost \eqref{eq:dp-metric}. Hence,
    \begin{align}
    \varepsilon_t^2=\E[\Wp^2(\Pi_t,\delta_{\theta^\star})]
    &\lesssim\sum_{j\ge1}\cbr{\frac{q_j}{1+tq_j}
    +\frac{p_j\theta_{\star,j}^2}{(1+tq_j)^2} }.
    \label{eq:sequence-bias-variance}
    \end{align}
For $t\ge1$, let $j_t:=\lceil t^{1/\zeta}\rceil$. Splitting at $j_t$ gives
    \begin{align*}
    \sum_{j\ge1}\frac{q_j}{1+tq_j}
    \lesssim \frac{j_t}{t}+\sum_{j>j_t}j^{-\zeta}
    \lesssim t^{-(\zeta-1)/\zeta}.
    \end{align*}
For the bias term, the Sobolev constraint yields
    \begin{align*}
    \sum_{j\ge1}\frac{p_j\theta_{\star,j}^2}{(1+tq_j)^2}
    \le R^2\sup_{j\ge1} \frac{p_jj^{-2\fs}}{(1+tq_j)^2}.
    \end{align*}
Because $p_jj^{-2\fs}\asymp j^{-(\zeta-1)}$ and $q_j\asymp j^{-\zeta}$, the last supremum is bounded by $t^{-(\zeta-1)/\zeta}$ up to a multiplicative constant by considering separately $j\le j_t$ and $j>j_t$. This proves \eqref{eq:sequence-contraction}.

It remains to prove \eqref{eq:sequence-benchmark}. Fix $m\in\bN$. For $j\le m$, use the comparator of \cref{lem:coordinate-comparator} with the deterministic scale $n=Tp_j$ and for $j>m$, retain the prior. The product comparator differs from the prior in only finitely many coordinates, so its KL divergence is the sum of the coordinatewise KL divergences. Moreover
    \begin{align*}
    \E[\ell_t(\theta, Y_t)]- \E[\ell_t(\theta^\star, Y_t)]
    =\sum_{j\ge1}p_jD_A(\theta_j,\theta_{\star,j}).
    \end{align*}
For $j\le m$, \cref{lem:coordinate-comparator} gives a risk contribution
$Tp_j/(\kappa_j+Tp_j)\lesssim 1$ and KL contribution
$\{1+\log(1+Tq_j)+\kappa_j\theta_{\star,j}^2\}$. For $j>m$, the quadratic upper bound on $D_A$ and \cref{lem:strong-convex-moment} applied to the prior at $\theta_0=0$ give
    \begin{align*}
    \int D_A(\theta,\theta_{\star,j})\Pi_{0,j}(\d\theta)
    \le C'_2 \{\kappa_j^{-1}+\theta_{\star,j}^2\}.
    \end{align*}
for any $j\in\bN$ for some absolute constant $C'_2>0$. Consequently, \cref{prop:gibbs-benchmark} yields
    \begin{align*}
    \fR\del[1]{(\Pi_t)_{t=0}^{T-1}}
    \lesssim m+\sum_{j\le m}\log(1+Tq_j)
    +\sum_{j\le m}\kappa_j\theta_{\star,j}^2
    +T\sum_{j>m}p_j\{\theta_{\star,j}^2+\kappa_j^{-1}\}.
\end{align*}
Choose $m=\lceil T^{1/\zeta}\rceil$. Since $T\le C m^\zeta$,
    \begin{align*}
    \sum_{j\le m}\log(1+Tq_j)
    \lesssim \sum_{j\le m}\{1+\zeta\log(m/j)\}
    \lesssim m,
    \end{align*}
where the last step follows from $\log(m!)\ge m\log m-m$. Also,
    \begin{align*}
    \sum_{j\le m}\kappa_j\theta_{\star,j}^2
    \lesssim m\sum_{j\le m}j^{2\fs}(\theta^{\star}_j)^2
    \lesssim mR^2.
    \end{align*}
Finally,
    \begin{align*}
    T\sum_{j>m}p_j\theta_{\star,j}^2
    \lesssim  TR^2\sup_{j>m}j^{-\fa-2\fs}
    \lesssim Tm^{1-\zeta},
    \end{align*}
and
    \begin{align*}
    T\sum_{j>m}p_j\kappa_j^{-1}
    \lesssim  T\sum_{j>m}j^{-\zeta}
    \lesssim  Tm^{1-\zeta}.
\end{align*}
Since $m\asymp T^{1/\zeta}$, all terms are $O(T^{1/\zeta})$, proving \eqref{eq:sequence-benchmark}.
\end{proof}

\subsection{Proof of \cref{prop:spectral-error}}

\begin{proof}
By \cref{lem:product-w2},
    \begin{align*}
    \Wp^2(Q_t^{(m)},\Pi_t)
    = \sum_{j>m}p_j\W^2(\Pi_{0,j},\Pi_{t,j}).
    \end{align*}
By the triangle inequality
    \begin{align*}
    \W^2(\Pi_{0,j},\Pi_{t,j})
    \le  2\W^2(\Pi_{0,j}, \delta_{\theta^\star_j})  
    + 2 \W^2( \delta_{\theta^\star_j},\Pi_{t,j}), 
    \end{align*}
and thus,
    \begin{align*}
    \Wp^2(Q_t^{(m)},\Pi_t)
    &\le 2 \sum_{j>m}p_j \W^2(\Pi_{0,j}, \delta_{\theta^\star_j})
    + 2\sum_{j>m}p_j\W^2(\delta_{\theta^\star_j},\Pi_{t,j})\\
    &\le 2 \sum_{j>m}p_j \W^2(\Pi_{0,j}, \delta_{\theta^\star_j})
    +2\Wp^2(\Pi_t, \delta_{\theta^\star}),
    \end{align*}
where the second term in the last line is bounded by $t^{-(\zeta-1)/\zeta}$ by \cref{thm:sequence-statistical}. Hence, it remains to bound the first term. \cref{lem:strong-convex-moment} applied to the prior at $0$ gives
    \begin{align*}
    \W^2(\Pi_{0,j}, \delta_{\theta^\star_j})
    & \le \int(\theta- \theta^\star_j)^2\Pi_{0,j}(\d\theta)\\
    &\le 2 \int\theta^2\Pi_{0,j}(\d\theta)+2(\theta^\star_j)^2\\
    &\le C' \kappa_j^{-1}+2(\theta^\star_j)^2
    \end{align*}
for some absolute constant $C'>0$, which implies that
    \begin{align*}
    \E[\Wp^2(Q_t^{(m)},\Pi_t)]
    \lesssim \sum_{j>m}p_j\{\kappa_j^{-1}+(\theta^\star_j)^2\}.
    \end{align*}
The prior term is $O(m^{1-\zeta})$ because $p_j\kappa_j^{-1}\asymp j^{-\zeta}$. For the truth term,
    \begin{align*}
    \sum_{j>m}p_j(\theta^\star_j)^2
    \le R^2\sup_{j>m}p_jj^{-2\fs}
    \lesssim  m^{-(\fa+2\fs)}
    = m^{1-\zeta},
 \end{align*}
which completes the proof.
\end{proof}

\subsection{Proof of \cref{thm:spectral-fast}}

\begin{proof}
\cref{thm:sequence-statistical} and \cref{prop:spectral-error} give
    \begin{align*}
    \varepsilon_t\lesssim t^{-(\zeta-1)/(2\zeta)},
    \qquad
    \alpha_t\lesssim m^{(1-\zeta)/2}+t^{-(\zeta-1)/(2\zeta)}.
    \end{align*}
\cref{thm:main-transfer} therefore yields
    \begin{align*}
    \fR\del[1]{\{Q^{(m)}_t\}_{t=0}^{T-1}}-\fR\del[1]{(\Pi_t)_{t=0}^{T-1}}
    &\lesssim \sbr{ m^{(1-\zeta)/2}\sum_{t=1}^{T-1}t^{-(\zeta-1)/(2\zeta)}+\sum_{t=1}^{T-1}t^{-(\zeta-1)/\zeta}+Tm^{1-\zeta}}\\
    &\lesssim \sbr{ T^{(\zeta+1)/(2\zeta)}m^{(1-\zeta)/2}+ T^{1/\zeta} +Tm^{1-\zeta}}
    \end{align*}
Combining this with \eqref{eq:sequence-benchmark} proves \eqref{eq:spectral-regret-general}; setting $m_T=\lceil T^{1/\zeta}\rceil$ gives \eqref{eq:spectral-fast-rate}.
\end{proof}

\subsection{Proof of \cref{thm:sequence_minimax}}

\begin{proof}
A predictor is allowed to observe $X_t$ and the past history before choosing a conditional density $\widehat p_t(\cdot|X_t,\cG_{t-1})$. Its regret is
    \begin{align*}
    \E\sbr{\sum_{t=1}^T
    \log\frac{p_{\star,X_t}(Y_t)}{\widehat p_t(Y_t| X_t,\cG_{t-1})}}.
    \end{align*}
Let $m=\lfloor cT^{1/\zeta}\rfloor$ for a fixed small $c>0$ and put $b_m^2:=c_0m^{-2\fs-1}$. Consider the uniform prior on the hypercube
    \begin{align*}
    \theta_j=\omega_jb_m,
    \quad j=m+1,\ldots,2m,
    \qquad \omega_j\in\{-1,1\},
\end{align*}
with all other coordinates zero. If $c_0$ is sufficiently small, every vertex belongs to $\ThetaSob$.   Since the prior distribution $P_\omega$ of $\omega$ does not depend on $X_{1:T}$, we have
    \begin{align*}
    I(\omega;Y_{1:T}\mid X_{1:T})
    &=\KL(P_{\omega, Y_{1:T}|X_{1:T}}\| P_{\omega}P_{Y_{1:T}|X_{1:T}})\\
    &=\E_{P_\omega}\sbr{\KL(P_{Y_{1:T}|\omega, X_{1:T}}\|P_{Y_{1:T}|X_{1:T}})}.
    \end{align*}
%Let $M(\d y_{1:T}\mid X_{1:T}):=\E_{P_\omega}[P_{Y_{1:T}|\omega, X_{1:T}}(\d y_{1:T})]$ be the Bayesian mixture law induced by the prior. Moreover, 
Let $\widehat{P}_{Y_{1:T}|X_{1:T}}$ be the joint law of $Y_{1:T}$ generated by the sequential predictive densities. Then we can see that the Bayes-average log-loss regret of the sequential predictor under the prior $P_\omega$ is at least the conditional mutual information $I(\omega;Y_{1:T}\mid X_{1:T})$, as
    \begin{align*}
    &\E_{P_\omega}\sbr{\KL(P_{Y_{1:T}|\omega, X_{1:T}}\|\widehat{P}_{Y_{1:T}|\omega, X_{1:T}})}\\
    &=\E_{P_\omega}\sbr{\KL(P_{Y_{1:T}|\omega, X_{1:T}}\|P_{Y_{1:T}|X_{1:T}})}
    +\KL(P_{Y_{1:T}|X_{1:T}}|\widehat{P}_{Y_{1:T}|X_{1:T}}))\\
    &\ge I(\omega;Y_{1:T}\mid X_{1:T}).
    \end{align*}
Conditional on $X_{1:T}$, the observations attached to different active coordinates depend on disjoint independent bits, hence
    \begin{align*}
    I(\omega;Y_{1:T}\mid X_{1:T})
    = \sum_{j=m+1}^{2m}I(\omega_j;Y^{(j)}_{1:N_j}\mid N_j),
    \end{align*}
where $N_j$ is the number of occurrences of coordinate $j$ and $Y^{(j)}_{1:N_j}$ is the collection of observations attached to coordinate $j$. For one active coordinate, let $P_+$ and $P_-$ denote the laws at natural parameters $b_m$ and $-b_m$. Their Bhattacharyya affinity is
    \begin{align*}
    \int\sqrt{\d P_+\d P_-}
    =\exp\left\{A(0)-\frac{A(b_m)+A(-b_m)}2\right\}
    \le \e^{-\underline Ib_m^2/2},
    \end{align*}
by the lower bound on $A''$. So for $n$ observations the affinity is at most $\e^{-n\underline Ib_m^2/2}$. Since $\TV(P,Q)=1-\int\min(\d P,\d Q)$ and $\min(p,q)\le\sqrt{pq}$ pointwise, we have
    \begin{align*}
    \TV(P_+^{\otimes n},P_-^{\otimes n})
    \ge 1-\e^{-n\underline Ib_m^2/2}.
    \end{align*}
Let $M:=(P_+^{\otimes n}+P_-^{\otimes n})/2$. Conditional on $N_j=n$, Pinsker's inequality (e.g. \citet[Section~2.4]{tsybakov2009nonparametric}) gives
    \begin{align*}
    I(\omega_j;Y^{(j)}_{1:n})
    =\frac12\KL(P_+^{\otimes n}\|M)+\frac12\KL(P_-^{\otimes n}\|M)
    \ge \frac{1}{2}\TV(P_+^{\otimes n},P_-^{\otimes n})^2.
    \end{align*}
For $j\in[m+1,2m]$, $p_j\asymp m^{-\fa}$ and $\mu_j:=Tp_j\asymp Tm^{-\fa}$. Since $\mu_jb_m^2\asymp Tm^{-\zeta}\asymp1$ and $\mu_j\asymp m^{2\fs+1}\to\infty$, the same exponential-Markov calculation as in \cref{lem:binomial-resolvents} gives $\Pr(N_j<\mu_j/2)\le \e^{-c_0\mu_j}$. Hence $\Pr(N_j\ge\mu_j/2)\gtrsim 1$ uniformly over the active coordinates for all sufficiently large $T$. On this event the preceding mutual-information lower bound is bounded below by a positive constant depending only on the fixed model constants. Thus
    \begin{align*}
    \E [I(\omega_j;Y^{(j)}_{1:N_j}\mid N_j)]\gtrsim 1,
    \end{align*}
uniformly over active $j$. Summing over the $m$ active coordinates yields Bayes-average regret at least $m\asymp T^{1/\zeta}$ up to a multiplicative constant. Since the maximum risk dominates Bayes average risk, the minimax regret is at least $T^{1/\zeta}$ up to a multiplicative constant.
\end{proof}

\section{Proofs for \cref{sec:gp}}\label{app:gp}

We introduce additional notation used in this section. For a Gaussian law $P$ on $L^2(\bP_X)$ with covariance kernel $\sK_P$, the covariance operator $\Sigma_P$ is defined as
    \begin{align*}
        \Sigma_P f(x')  =\int \sK_P(x,x')f(x)\bP_X(\d x).
    \end{align*}
Let $\bH$ be the reproducing kernel Hilbert space (RKHS) associated with the prior GP $\Pi_0$. Moreover, we define
    \begin{equation}\label{eq:gp-cutoffs}
    J_t^\dag:=\ceil{t^{d/(d+2\fs)}},
    \qquad
    N_t^\dag:=\ceil{t^{d/(2\fs)}}
    \end{equation}

\subsection{Regular spectral design}
Define the event
    \begin{align}
    \label{eq:event_gram_matrix}
        \cE_t:=\cbr{\frac t2I_{N_t^\dag}
    \preceq \Psi_{t,N_t^\dag}^\top\Psi_{t,N_t^\dag}
    \preceq \frac{3t}{2}I_{N_t^\dag}}.
    \end{align}
 
\begin{lemma}[Regular spectral design]
\label{lemma}{lem:gp-design}
Under \cref{ass:gp-true}, there are absolute constants $c>0$ and $C>0$ such that
    \begin{equation}\label{eq:gp-design-bound}
    \Pr\del{ \cE_t}
    \ge 1- C N_t^\dag\exp\{-ct/N_t^\dag\}.
    \end{equation}
Since $N_t^\dag\asymp t^{d/(2\fs)}$, the probability of the complement tends to 0 faster than $t^{-b}$ for any $b>0$.
\end{lemma}

\begin{proof}
Let $N:=N_t^\dag$. For $\psi_{1:N}(x)=(\psi_1(x),\ldots,\psi_L(x))^\top$, orthonormality gives $\E[\psi_{1:N}(X)\psi_{1:N}(X)^\top]=I_L$, while \eqref{eq:gp-basis} gives $\opnorm{\psi_{1:N}(X)\psi_{1:N}(X)^\top}=\norm{\psi_{1:N}(X)}^2\le B_\psi^2L$.
The upper and lower matrix Chernoff inequalities with relative deviation $1/2$ yield \eqref{eq:gp-design-bound}.
\end{proof}

\subsection{Proof of \cref{thm:gp-predictive-comparison}}

\begin{lemma}[Gaussian predictive excess]
\label[lemma]{lem:gp-pred-formula}
Let $P$ be a Gaussian law on $L^2(\bP_X)$. Let $m_P$ be the mean function of $P$ and $v_P$ be the variance of $f(x)$ for $f\sim P$. Then
    \begin{equation}\label{eq:gp-pred-excess}
    \cM_t(P)-\cM_t(\delta_{f^\star})
    =\frac12\int \sbr{ g\del[3]{\frac{v_P(x)}{\sigma_\xi^2}}
    +\frac{\{m_P(x)-f^\star(x)\}^2}{\sigma_\xi^2+v_P(x)} }\bP_X(\d x),
    \end{equation}
where $g(u)=\log(1+u)-u/(1+u)$. In particular,
    \begin{equation}\label{eq:gp-pred-excess-crude}
    0\le \cM_t(P)-\cM_t(\delta_{f^\star})
    \le\frac1{2\sigma_\xi^2}\W^2(P,\delta_{f^\star}).
    \end{equation}
\end{lemma}

\begin{proof}
Conditionally on $X=x$, the true response law is $\N(f^\star(x),\sigma_\xi^2)$ and the predictive law is $\N(m_P(x),\sigma_\xi^2+v_P(x))$. Their Gaussian KL divergence gives \eqref{eq:gp-pred-excess}. Since $0\le g(u)\le u$ and $(\sigma_\xi^2+v_P)^{-1}\le\sigma_\xi^{-2}$, the right-hand side of  \eqref{eq:gp-pred-excess} is at most $\{\norm{m_P-f^\star}^2+\|v_P\|\}/(2\sigma_\xi^2)$, which equals the right-hand side of \eqref{eq:gp-pred-excess-crude}.
\end{proof}

\begin{lemma}[Trace norm and Bures–Wasserstein distance]
\label[lemma]{lem:gp-bures}
For positive trace-class operators $A$ and $B$ 
    \begin{equation}\label{eq:gp-bures}
    \norm{A-B}_1
    \le \{\sqrt{\tr(A)}+\sqrt{\tr(B)}\}
    \sd_{\rm BW}(A,B),
    \end{equation}
where $\|\cdot\|_1$ denotes the Schatten-1 norm and $\sd_{\rm BW}$ does the Bures–Wasserstein distance defined as
    \begin{align*}
    \sd_{\rm BW}^2(A,B)
    :=\tr(A+B-2(A^{1/2}BA^{1/2})^{1/2}).
    \end{align*}
\end{lemma}

\begin{proof}
Use the Hilbert--Schmidt representation $\sd_{\rm BW}(A,B)=\inf_U\norm{A^{1/2}-B^{1/2}U}_{\rm HS}$. For $X=A^{1/2}$ and $Y=B^{1/2}U$, $A-B=(X-Y)X^\ast+Y(X^\ast-Y^\ast)$. By the H\"older inequality for the Schatten norm gives
    \begin{align*}
    \norm{A-B}_1
    \le\{\norm{X}_{\rm HS}+\norm{Y}_{\rm HS}\}
    \norm{X-Y}_{\rm HS}.
    \end{align*}
Taking the infimum over $U$ proves the claim.
\end{proof}

\begin{proof}[Proof of \cref{thm:gp-predictive-comparison}]
Recall the notation used in \cref{lem:gp-pred-formula,lem:gp-bures}. Let $m_P$ and $\Sigma_P$ be the mean function and covariance operator of $P$. Define $m_Q$ and $\Sigma_Q$ likewise. We write
    \begin{align*}
    \varepsilon:=\W(P,\delta_{f^\star}),
    \qquad  \alpha:=\W(P,Q),
    \end{align*}
and put $b_P:=m_P-f^\star$ and $b_Q:=m_Q-f^\star$. By \cref{lem:gp-pred-formula}, we have
    \begin{align}
    \cM_t(Q)-\cM_t(P)
    &=\cM_t(Q)-\cM_t(\delta_{f^\star})-\{\cM_t(P)-\cM_t(\delta_{f^\star})\} \notag\\
    &=\frac12\int
    \left[ g\left(\frac{v_Q(x)}{\sigma_\xi^2}\right)-g\left(\frac{v_P(x)}{\sigma_\xi^2}\right)
    +\frac{b_Q(x)^2}{\sigma_\xi^2+v_Q(x)}-\frac{b_P(x)^2}{\sigma_\xi^2+v_P(x)}
    \right]\bP_X(\d x) \label{eq:gauss-pred-diff},
    \end{align}
By the Gaussian Wasserstein formula $\W^2(P,Q)=\|m_p-m_Q\|^2+\sd_{\rm BW}^2(\Sigma_Q,\Sigma_P)$, we have $\sd_{\rm BW}(\Sigma_Q,\Sigma_P)\le \alpha$. Therefore, \cref{lem:gp-bures} gives
    \begin{align*}
    \int |v_Q(x)-v_P(x)|\bP_X(\d x)
    \le\norm{\Sigma_Q-\Sigma_P}_1
    \le(2\varepsilon+\alpha)\alpha,
    \end{align*}
because $\sqrt{\tr(\Sigma_P)}\le \varepsilon$ and $\sqrt{\tr(\Sigma_Q)}\le \varepsilon+\alpha$. In \eqref{eq:gp-pred-excess}, $g'(u)=u/(1+u)^2\le1/4$, so the variance-only difference in \eqref{eq:gauss-pred-diff} is bounded by a constant multiple of $(2\varepsilon+\alpha)\alpha$. For the mean-related difference in \eqref{eq:gauss-pred-diff}, we use the inequality
    \begin{align*}
    &\int \abs{\frac{b_Q^2(x)}{\sigma_\xi^2+v_Q(x)}-\frac{b_P^2(x)}{\sigma_\xi^2+v_P(x)}}\bP_X(\d x)\\
    &\le\frac{1}{\sigma_\xi^2}\int \abs{b_Q^2(x)-b_P^2(x)}\bP_X(\d x)
    +\frac{1}{\sigma_\xi^4}\int b_P^2(x)|v_Q(x)-v_P(x)|\bP_X(\d x)\\
    &\le\frac{1}{\sigma_\xi^2}\int \abs{b_Q^2(x)-b_P^2(x)}\bP_X(\d x)
    +\frac{1}{\sigma_\xi^4}\overline v(P,Q)\norm{b_P}^2
    \end{align*}
For the first term, we use the Gaussian Wasserstein formula to have $\norm{b_Q-b_P}=\norm{m_Q-m_P}\le \alpha$ and $\norm{b_P}\le \varepsilon$. Moreover, $\norm{b_Q}\le \norm{b_Q-b_P}+\norm{b_P}\le \varepsilon+\alpha$. Therefore, by the Cauchy--Schwarz inequality
    \begin{equation}\label{eq:gp-mean-square-diff}
    \int \abs{b_Q^2(x)-b_P^2(x)}\bP_X(\d x)
    \le  \norm{b_Q-b_P}( \norm{b_Q}+  \norm{b_P})
    \le \alpha(2\varepsilon+\alpha).
    \end{equation}
Combining these bounds proves \eqref{eq:gp-predictive-comparison}.
\end{proof}

\subsection{Proof of \cref{thm:gp-statistical}}

\begin{proof}
For a Gaussian process posterior $\Pi_t$, we have
    \begin{align*}
    \W^2(\Pi_t,\delta_{f^\star})
    =\norm[0]{\widehat f_t-f^\star}^2
    +\tr(\Sigma_t).
\end{align*}
where $\widehat f_t$ denotes the posterior mean. Taking expectations and applying \cref{lem:gp-krr,prop:gp-covariance} below proves \eqref{eq:gp-contraction}. \cref{lem:gp-pred-formula} gives
    \begin{align*}
    \fR\del[1]{(\Pi_t)_{t=0}^{T-1}}
    \le \frac1{2\sigma_\xi^2}\sum_{t=0}^{T-1} \E[\W^2(\Pi_t,\delta_{f^\star})].
    \end{align*}
The $t=0$ term is finite and the remaining sum is bounded by $\sum_{t=1}^{T-1}t^{-2\fs/(d+2\fs)}
\lesssim T^{d/(d+2\fs)}$ up to a multiplicative constant.
\end{proof}

\begin{lemma}[Exact posterior covariance geometry]
\label[lemma]{prop:gp-covariance}
Let $\Sigma_t$ be the covariance operator of the GP posterior $\Pi_t$. Under \cref{ass:gp-true,ass:gp-eigenvalues},
    \begin{equation}\label{eq:gp-cov-op-app}
    \E\sbr{\opnorm{\Sigma_t}}\lesssim t^{-1},
    \qquad
    \E\sbr{\tr(\Sigma_t)} \lesssim t^{-2\fs/(d+2\fs)}.
    \end{equation}
\end{lemma}

\begin{proof}
Put $N=N_t^\dag \asymp t^{d/(2\fs)}$ and let $\sP_N$ denote the orthogonal projection onto $\cH_N:= \operatorname{span}(\psi_1,\ldots,\psi_N).$ Under the prior GP, 
    \begin{align}
    u_{1:N}
    :=  \left( \inn{f}{\psi_1}, \ldots,\inn{f}{\psi_N}\right)^\top 
    \sim\N(0,\Lambda_N),
    \end{align}
with $\Lambda_N:=\operatorname{diag}(\lambda_1,\ldots,\lambda_N)$. We first control the contribution of the coefficients above $N$. Define 
    \begin{equation}\label{eq:gp-residual}
    R_{t,N}:=\sum_{j>N}\lambda_j\psi_j(X_{1:t})\psi_j(X_{1:t})^\top,
    \end{equation}
where $\psi_j(X_{1:t}):=(\psi_j(X_1),\ldots,\psi_j(X_t))^\top$. By uniform boundedness of the eigenfunctions in \eqref{eq:gp-basis} and the eigenvalue decay in \cref{ass:gp-eigenvalues},
    \begin{align}
    \opnorm{R_{t,N}}
    &\le \tr(R_{t,N})
    =  \sum_{i=1}^t\sum_{j>N}  \lambda_j\psi_j(X_i)^2 \notag\\
    &\le  B_\psi^2t\sum_{j>N}\lambda_j
    \lesssim tN^{-2\fs/d}
    \lesssim 1.
    \label{eq:gp-tail-L}
    \end{align}
To identify the posterior covariance of the first $N$ coefficients, write
    \begin{align}
    Y_{1:t} = \Psi_{t,N}u_{1:N} +\eta_{t,N} +\xi_{1:t},
    \quad  \eta_{t,N} := \sum_{j>N}u_j\psi_j(X_{1:t}).
    \end{align}
Note that
    \begin{align}
    \eta_{t,N}|X_{1:t}\sim\N(0,R_{t,N}),
    \qquad
    \xi_{1:t}\sim\N(0,\sigma_\xi^2I_t),
\end{align}
and both vectors are independent of $ u_{1:N}$. Hence, after integrating out the coefficients above $N$, we have
    \begin{align}
    Y_{1:t}\mid u_{1:N},X_{1:t}
    \sim    \N\del[0]{\Psi_{t,N} u_{1:N}, \sigma_\xi^2I_t+R_{t,N}}.
\end{align}
Gaussian conjugacy therefore shows that the posterior covariance matrix of $u_{1:N}$ conditional on $X_{1:t}$, equivalently the matrix representation of $\sP_N\Sigma_t\sP_N$ in the basis $(\psi_j)_{j\le N}$, is
    \begin{equation}\label{eq:gp-head-cov}
    \Sigma_{t,N}
    = \cbr[1]{\Lambda_N^{-1}+ \Psi_{t,N}^\top (\sigma_\xi^2I_t+R_{t,N})^{-1} \Psi_{t,N} }^{-1}.
    \end{equation}
By \eqref{eq:gp-tail-L}, there exists an absolute constant $C'>0$ such that
    \begin{align}
    \sigma_\xi^2I_t+R_{t,N}
    \preceq  (\sigma_\xi^2+C')I_t,
    \end{align}
and hence
    \begin{align}
    (\sigma_\xi^2I_t+R_{t,N})^{-1}
    \succeq  \frac{1}{\sigma_\xi^2+C'}I_t.
    \end{align}
Since $ \Psi_{t,N}^\top\Psi_{t,N}  \succeq  \frac t2I_N$ on $\cE_t$ define in \eqref{eq:event_gram_matrix}, we have 
    \begin{align}
    \Psi_{t,N}^\top  (\sigma_\xi^2I_t+R_{t,N})^{-1}  \Psi_{t,N}
    \succeq  \frac{t}{2(\sigma_\xi^2+C')}I_N
    \end{align}
    on $\cE_t$.
Since the prior precision $\Lambda_N^{-1}$ is positive semidefinite, \eqref{eq:gp-head-cov} gives $ \opnorm{\Sigma_{t,N}}\ind_{\cE_t}\lesssim t^{-1}$. 

It remains to pass from the first $N$ coefficient directions to the full covariance operator. Relative to the orthogonal decomposition $L^2(\bP_X)  =  \cH_N\oplus\cH_N^\perp,$ write
    \begin{align}
    \Sigma_t  =
    \begin{pmatrix}
    A & B\\
    B^\ast & D
    \end{pmatrix},
    \end{align}
where $A=\sP_N\Sigma_t\sP_N$ and $D=\sP_{>N}\Sigma_t\sP_{>N}$, with the orthogonal projection $\sP_{>N}$ onto $\operatorname{span}(\psi_{j}:j>N)$. The matrix representation of $A$ is
$\Sigma_{t,N}$, so $  \opnorm{A}\ind_{\cE_t}\lesssim t^{-1}$. Moreover, since Gaussian conditioning reduces covariance, we have $  0\preceq\Sigma_t\preceq\Sigma_0$ and so
    \begin{align}
    0\preceq D  \preceq  \sP_{>N}\Sigma_0\sP_{>N}.
    \end{align}
Hence,
    \begin{align}
    \label{eq:op_d_bound}
    \opnorm{D} \le  \sup_{j>N}\lambda_j
    \lesssim N^{-1-2\fs/d}
    \lesssim t^{-1}.
    \end{align}
Because $\Sigma_t$ is positive semidefinite, covariance Cauchy--Schwarz gives, for
$x\in\cH_N$ and $y\in\cH_N^\perp$,
    \begin{align}
    |\inn{x}{By}|
    \le  \sqrt{\inn{x}{Ax}}\sqrt{\inn{y}{Dy}}.
    \end{align}
Thus
    \begin{align*}
    \inn{(x,y)}{\Sigma_t(x,y)}
    &\le \cbr{ \sqrt{\opnorm{A}}\norm{x} + \sqrt{\opnorm{D}}\norm{y} }^2\\
    &\le  \{\opnorm{A}+\opnorm{D}\}(\norm{x}^2 + \norm{y}^2).
    \end{align*}
Taking the supremum over unit vectors yields $\opnorm{\Sigma_t} \le  \opnorm{A}+\opnorm{D}$. Thus, by \cref{lem:gp-design},  \cref{eq:op_d_bound} and the fact that  $  0\preceq\Sigma_t\preceq\Sigma_0$, we have
    \begin{align}
         \E\sbr{\opnorm{\Sigma_t}}
         &\le \E\sbr{\opnorm{\Sigma_t}1_{\cE_t}} + \E\sbr[1]{\opnorm{\Sigma_t}\ind_{\cE_t^\complement}}\\
         &\le \E\sbr{\opnorm{A}\ind_{\cE_t}}+ \E\sbr{\opnorm{D}} +\lambda_1 \Pr(\cE_t^\complement)
         \lesssim t^{-1}.
    \end{align}
which proves the operator-norm bound.

We now bound the trace. Put $J=J_t^\dag$. Since $\sP_J\Sigma_t \sP_J$ is positive semidefinite and has rank at most $J$,
    \begin{align}
    \tr(\sP_J\Sigma_t \sP_J)\ind_{\cE_t}
    \le  J\opnorm{\Sigma_t}\ind_{\cE_t}
    \lesssim \frac{J}{t}.
\end{align}
For the remaining directions, using covariance domination  $  0\preceq\Sigma_t\preceq\Sigma_0$ again, we have
    \begin{align*}
    \tr(\sP_{> J}\Sigma_t\sP_{> J} )
    \le  \tr(\sP_{> J}\Sigma_0 \sP_{> J})
    = \sum_{j>J}\lambda_j
    \lesssim J^{-2\fs/d}.
    \end{align*}
Combining these two inequalities, for $J_t^\dag \asymp t^{d/(d+2\fs)}$,  we have
    \begin{align}
    \tr(\Sigma_t)\ind_{\cE_t}
    \lesssim  \cbr{ \frac{J}{t} + J^{-2\fs/d}}
    \asymp  t^{-2\fs/(d+2\fs)}.
    \end{align}
Therefore, we have
 \begin{align}
         \E\sbr{\tr(\Sigma_t)}
         &\le \E\sbr{\tr(\Sigma_t)\ind_{\cE_t}} + \E\sbr[1]{\tr(\Sigma_t)\ind_{\cE_t^\complement}}\\
         &\lesssim t^{-2\fs/(d+2\fs)}+ \tr(\Sigma_0) \Pr(\cE_t^\complement)
         \lesssim t^{-2\fs/(d+2\fs)}
    \end{align}
which completes the proof.
\end{proof}

\begin{lemma}[Risk of the exact posterior mean]\label[lemma]{lem:gp-krr}
Let $\widehat f_t$ denote the mean of the GP posterior $\Pi_t$.
 Under \cref{ass:gp-true,ass:gp-eigenvalues},
    \begin{equation}\label{eq:gp-krr-rate}
    \E\sbr{\norm[0]{\widehat f_t-f^\star}^2}
    \lesssim t^{-2\fs/(d+2\fs)}.
    \end{equation}
\end{lemma}

\begin{proof}
The exact GP posterior mean is the kernel-ridge regression estimator
    \begin{align*}
    \widehat f_t
    = \argmin_{f\in\bH}
    \cbr{ \frac1t\sum_{i=1}^t\{Y_i-f(X_i)\}^2 +\rho_t\norm{f}_{\bH}^2},
    \qquad  \rho_t:=\frac{\sigma_\xi^2}{t}.
    \end{align*}
We apply Theorem~1(ii) of \citet{fischer2020sobolev}. In their notation, set
    \begin{align*}
    p:=\frac{d}{d+2\fs},
    \qquad  \beta:=\frac{2\fs}{d+2\fs},
    \end{align*}
so that $p+\beta=1$.  The eigenvalue condition in \cref{ass:gp-eigenvalues} is then written as $\lambda_j\asymp j^{-1/p}.$ Moreover,
    \begin{align*}
    \sum_{j\ge1} \lambda_j^{-\beta}(f_j^\star)^2
    \lesssim  \sum_{j\ge1} j^{2\fs/d}(f_j^\star)^2
    <\infty,
    \end{align*}
so the source condition of \citet{fischer2020sobolev} holds with exponent $\beta$. The uniform bound on the eigenfunctions implies that, for every $q>p$,
    \begin{align*}
      \sup_{x\in\mathcal X}
      \sum_{j\ge1}\lambda_j^q\psi_j(x)^2
      \le B_\psi^2\sum_{j\ge1}\lambda_j^q
      <\infty.
      \end{align*}
which verifies their embedding condition  with any fixed $q\in(p,1)$, and that
     \begin{align}
     \label{eq:f-star-bound}
     \|f^\star\|_\infty
     \le B_\psi\sum_{j\ge1} |f_j^\star|
     \le B_\psi R\del[2]{\sum_{j\ge1} j^{-2\fs/d}}^{1/2}<\infty.
    \end{align}
Moreover, Gaussian noise satisfies their moment condition. Since
    \begin{align*}
    \beta+p=1>q
    \qquad\text{and}\qquad
    \rho_t\asymp t^{-1/(\beta+p)}=t^{-1},
    \end{align*}
Theorem~1(ii) of \citet{fischer2020sobolev} implies that there is  an absolute constant $C'>0$ such that
    \begin{equation}\label{eq:gp-krr-hp}
    \Pr\del{ \norm[1]{\widehat f_t-f^\star}^2 > C'\tau^2 t^{-\beta}}
    \le 4\e^{-\tau}
\end{equation}
for $\tau\ge1$. We now pass from \eqref{eq:gp-krr-hp} to an expectation bound. Inspection of the sample-size condition in \citet[Theorem~16 and the proof of Theorem~1]{fischer2020sobolev}
shows that \eqref{eq:gp-krr-hp} holds uniformly for
    \begin{align*}
      1\le \tau\le\bar\tau_t,
      \qquad  \bar\tau_t  :=c\frac{t^{1-q}}{\log t},
    \end{align*}
for some sufficiently small constant $c>0$ and all sufficiently large $t$. For simplicity, we let $W_t:=\norm[1]{\widehat f_t-f^\star}^2.$ By the change of variables $v^2=(C't^{-\beta})^{-1}u$, the truncated expectation is bounded as 
    \begin{align*}
    \E\sbr{W_t\wedge C'\bar\tau_t^2t^{-\beta}}
    &\le C't^{-\beta} + \int_{C' t^{-\beta}}^{C'\bar\tau_t^2 t^{-\beta} }\Pr(W_t>u)\d u \\
    &\le C't^{-\beta} + 2(C't^{-\beta}) \int_1^{\bar\tau_t}v\Pr(W_t>(C't^{-\beta})v^2)\d v\\
    &\le C't^{-\beta} + 8(C't^{-\beta}) \int_1^{\bar\tau_t}v\e^{-v}\d v
    \lesssim t^{-\beta}.
\end{align*}
It remains to control the upper tail. By optimality of the kernel-ridge estimator and comparison with $f=0$,
    \begin{align*}
    \rho_t\norm[1]{\widehat f_t}_{\bH }^2
    \le   \frac1t\sum_{i=1}^tY_i^2.
\end{align*}
Since $\rho_t=\sigma_\xi^2/t$ and $\norm{f}^2 \le\lambda_1\norm{f}_{\bH}^2$, Gaussian noise and the boundedness of $f^\star$ imply
    \begin{align*}
    \E [W_t^2]\lesssim t^2.
    \end{align*}
Furthermore, \eqref{eq:gp-krr-hp} at $\tau=\bar\tau_t$ gives
    \begin{align*}
    \Pr\{W_t>C'\bar\tau_t^2t^{-\beta}\}
    \le4\e^{-\bar\tau_t}.
\end{align*}
Therefore, by the Cauchy--Schwarz inequality
    \begin{align*}
    \begin{aligned}
    \E\sbr{W_t\ind\{W_t>C'\bar\tau_t^2t^{-\beta}\}}
    &\le(\E W_t^2)^{1/2}\Pr\{W_t>C'\bar\tau_t^2t^{-\beta}\}^{1/2} \\
    &\lesssim t \e^{-\bar\tau_t/2}
    = o(t^{-\beta}).
    \end{aligned}
    \end{align*}
Combining the truncated expectation and the upper-tail bound gives
    \begin{align*}
    \E [W_t]\lesssim t^{-\beta}= t^{-2\fs/(d+2\fs)}.
    \end{align*}
Finitely many small values of $t$ are absorbed into the constant.
\end{proof}

\subsection{Proof of \cref{thm:sparse-gp-accuracy}}

\begin{proof}
The variational process posterior is absolutely continuous with respect to the exact posterior and has finite reverse KL. Conditional on the data, \cref{lem:gp-transport} below gives
    \begin{align*}
    \W^2(Q_t^{(J)},\Pi_t)
    \le2\opnorm{\Sigma_t}\KL(Q_t^{(J)}\|\Pi_t).
    \end{align*}
In the proof of \cref{prop:gp-covariance}, we have shown that on the event $\cE_t$ in \eqref{eq:event_gram_matrix}, $\opnorm{\Sigma_t}\lesssim t^{-1}$. Therefore, \cref{lem:gp-KL} gives
    \begin{align*}
       \E\sbr{\opnorm{\Sigma_t}\KL(Q_t^{(J)}\|\Pi_t)\ind_{\cE_t}} 
       \lesssim t^{-1} \E\sbr{\KL(Q_t^{(J)}\|\Pi_t)} 
       \lesssim t^{-2\fs/(d+2\fs)}.
    \end{align*}
On the complement $\cE_t^\complement$, since $\opnorm{\Sigma_t}\le \opnorm{\Sigma_0}\lesssim 1$ and the event $\cE_t^\complement$ does not depend on the outcome variables $Y_{1:t}$, we have
    \begin{align*}
        \E\sbr{\opnorm{\Sigma_t}\KL(Q_t^{(J)}\|\Pi_t)\ind_{\cE_t^\complement}} 
        &\lesssim \E\sbr{\KL(Q_t^{(J)}\|\Pi_t)\ind_{\cE_t^\complement}} \\
        &=\E\sbr{\E\sbr{\KL(Q_t^{(J)}\|\Pi_t)\big|X_{1:t}}\ind_{\cE_t^\complement}} \\
        &\lesssim t\Pr(\cE_t^\complement)=o(t^{-2\fs/(d+2\fs)}),
    \end{align*}
where we use \cref{lem:gp-KL-clude} given below for the second inequality. Combining the derived inequalities, we get the desired result.
\end{proof}

\begin{lemma}[Gaussian transportation]\label[lemma]{lem:gp-transport}
Let $\gamma$  be a Gaussian law on $L^2(\bP_X)$ with covariance operator $\Sigma$. If $\nu\ll\gamma$ and $\KL(\nu\|\gamma)<\infty$, then
    \begin{equation}\label{eq:gp-T2}
    \W^2(\nu,\gamma) \le 2\opnorm{\Sigma}\KL(\nu\|\gamma).
    \end{equation}
\end{lemma}

\begin{proof}
By translation invariance of both the Wasserstein distance and relative
entropy, it is enough to consider the centered case. Let $(e_j)_{j\ge1}$ be an orthonormal basis corresponding to the positive eigenvalues of $\Sigma$. Let $H_\Sigma:=  \overline{\operatorname{span}(e_j:j\ge 1)}.$ Then the Gaussian measure $\gamma$ is supported on $H_\Sigma$, and $\nu\ll\gamma$ implies that $\nu$ is supported on the same space. For $N\ge1$, let $\sP_N$ be the orthogonal
projection onto $\operatorname{span}(e_1,\ldots,e_N)$, and write
    \begin{align*}
    \nu_N:=(\sP_N)_{\#}\nu,
    \qquad  \gamma_N:=(\sP_N)_{\#}\gamma.
    \end{align*}
Then
    \begin{align*}
    \gamma_N = \N(0,\Sigma_N),
    \qquad\Sigma_N:=\sP_N\Sigma \sP_N.
\end{align*}
The finite-dimensional Gaussian transportation inequality \citep{talagrand1996transportation}, followed by linear rescaling, gives
    \begin{align*}
    \W^2(\nu_N,\gamma_N)
    \le  2\opnorm{\Sigma_N}\KL(\nu_N\|\gamma_N).
\end{align*}
Since $ \opnorm{\Sigma_N}\le\opnorm{\Sigma}$ and $\KL(\nu_N\|\gamma_N) \le \KL(\nu\|\gamma)$ (because the KL divergence decreases under measurable transformations), we have, for every $N$,
    \begin{equation}\label{eq:projected-gaussian-T2}
    \W^2(\nu_N,\gamma_N)
    \le  2\opnorm{\Sigma}\KL(\nu\|\gamma).
\end{equation}
Since $\gamma$ is Gaussian, Fernique's theorem gives some $\eta>0$ such that
    \begin{align*}
    \int \exp\{\eta\norm{x}^2\}\gamma(\d x)<\infty.
    \end{align*}
The Donsker-Varadhan variational formula therefore yields
    \begin{align*}
    \eta\int \norm{x}^2\nu(\d x)
    \le  \KL(\nu\|\gamma)
    +\log\int \exp(\eta\norm{x}^2)\gamma(\d x)
    <\infty.
\end{align*}
Thus both $\nu$ and $\gamma$ have finite second moments. Using the coupling $x\mapsto(x,\sP_Nx)$,
    \begin{align*}
    \W^2(\nu,\nu_N)
    \le  \int \norm{(I-\sP_N)x}^2\nu(\d x)
    \longrightarrow0,
\end{align*}
due to dominated convergence, because $\sP_Nx\to x$ as $N\to\infty$ for every $x\in H_\Sigma$ and the integrand is bounded by $\norm{x}^2$. The same argument gives $\W(\gamma,\gamma_N)\longrightarrow0.$ Hence, by the triangle inequality,
    \begin{align*}
    \W(\nu_N,\gamma_N)\longrightarrow \W(\nu,\gamma).
    \end{align*}
Letting $N\to\infty$ in \eqref{eq:projected-gaussian-T2} proves
    \begin{align*}
    \W^2(\nu,\gamma)
    \le 2\opnorm{\Sigma}\KL(\nu\|\gamma).
\end{align*}
If $\Sigma=0$, then $\gamma=\delta_0$, and $\nu\ll\gamma$ implies $\nu=\gamma$, so the conclusion is immediate.
\end{proof}

\begin{lemma}[KL approximation error of sparse GP]\label[lemma]{lem:gp-KL}
Under \cref{ass:gp-true,ass:gp-eigenvalues}, if $J\ge J_t^\dag$, then
    \begin{align}
    \E[\KL(Q_t^{(J)}\|\Pi_t)]&\lesssim t^{d/(d+2\fs)}.\label{eq:gp-KL-rate}
    \end{align}
\end{lemma}

\begin{proof}
We apply Lemma 3 of \citet{nieman2022contraction}: for every $h\in \bH$,
    \begin{equation}\label{eq:gp-nieman}
    \E[\KL(Q_t^{(J)}\|\Pi_t)]
    \le \frac1{\sigma_\xi^2}
    \sbr{ t\norm{f^\star-h}^2 +\norm{h}_{\bH}^2\E[\opnorm{R_{t,J}}] +\E[\tr(R_{t,J})]}
\end{equation}
where $R_{t,J}:=\sum_{j>J}\lambda_j\psi_j(X_{1:t})\psi_j(X_{1:t})^\top$. By uniform boundedness of the eigenfunctions in \eqref{eq:gp-basis} and the eigenvalue decay in \cref{ass:gp-eigenvalues}, we have the trace bound
    \begin{align}
    \tr(R_{t,J})
    =  \sum_{i=1}^t\sum_{j>J}  \lambda_j\psi_j(X_i)^2 
    &\le  B_\psi^2t\sum_{j>J}\lambda_j
    \lesssim tJ^{-2\fs/d}.
    \label{eq:gp-res-trace}
    \end{align}
Moreover, by applying Lemma 5 of \citet{nieman2022contraction}, we have
    \begin{align*}
    \E[\opnorm{R_{t,J}}]
    \lesssim 1+tJ^{-1-2\fs/d}
    +t^{d/(2\fs)}J^{-2\fs/d}\log t.
    \end{align*}
At $J\ge J_t^\dag$ this is bounded by a constant because $\fs>d$. Now we take $h=\sP_{J_t^\dag}f^\star$, the orthogonal projection onto $\operatorname{span}(\psi_1,\ldots,\psi_J)$. Then \eqref{eq:gp-truth} in \cref{ass:gp-true} gives
    \begin{align*}
    \norm{f^\star-h}^2\le C(J_t^\dag)^{-2\fs/d},
    \qquad
    \norm{h}_{\bH}^2
    =\sum_{j\le J_t^\dag}\frac{(f_j^\star)^2}{\lambda_j}
    \lesssim J_t^\dag.
    \end{align*}
Substitution of these bounds into \eqref{eq:gp-nieman} proves \eqref{eq:gp-KL-rate}.
\end{proof}

\begin{lemma}[Crude KL approximation error of sparse GP]
\label[lemma]{lem:gp-KL-clude}
Under \cref{ass:gp-true,ass:gp-eigenvalues},
    \begin{align}
    \E[\KL(Q_t^{(J)}\|\Pi_t)|X_{1:t}]&\lesssim t
    \end{align}
for any given covariates $X_{1:t}:=(X_1,\dots, X_t)$.
\end{lemma}

\begin{proof}
By inspecting the proof of Lemma 3 of \citet{nieman2022contraction}, we have the conditional expectation version:
    \begin{equation}
    \E[\KL(Q_t^{(J)}\|\Pi_t)|X_{1:t}]
    \lesssim \sum_{i=1}^t(f^\star(X_i)-h(X_i))^2 +\norm{h}_{\bH}^2\opnorm{R_{t,J}} +\tr(R_{t,J})
\end{equation}
where $R_{t,J}:=\sum_{j>J}\lambda_j\psi_j(X_{1:t})\psi_j(X_{1:t})^\top$. Take a constant function $h=0$. Then the desired result follows from combining two observations:  $\|f^\star\|_\infty$ is bounded under our assumption as we have shown in \eqref{eq:f-star-bound}, and the trace term is bounded as $ \tr(R_{t,J})\lesssim t$  as we have shown in \eqref{eq:gp-res-trace}.
\end{proof}

\subsection{Proof of \cref{thm:gp-fast}}

\begin{lemma}[Uniformly bounded predictive variance]
\label[lemma]{lem:gp-pointvar}
Assume \eqref{eq:gp-basis}. Then  for every $t\ge0$ and every $J\ge1$,
    \begin{equation}\label{eq:gp-pointvar}
    \sup_{x\in\mathcal X}v_{\Pi_t}(x)
    +\sup_{x\in\mathcal X}v_{Q_t^{(J)}}(x)
    \lesssim 1.
    \end{equation}
\end{lemma}

\begin{proof}
The exact posterior covariance kernel is
    \begin{align*}
    \sK_t(x,x')= \sK_0(x,x')
    -\sK_0(x, X_{1:t})(\sigma_\xi^2 I + \sK_0(X_{1:t},X_{1:t}))^{-1} \sK_0( X_{1:t},x').
    \end{align*}
Consequently,
    \begin{align*}
    0\le v_{\Pi_t}(x)=\sK_t(x,x)\le \sK_0(x,x).
    \end{align*}
For the variational posterior, \eqref{eq:gp-q-covariance} gives
    \begin{align*}
    v_{Q_t^{(J)}}(x)
    =\sK_t^{(J)}(x,x)
    =\sK_0(x,x)-\psi_{1:J}(x)^\top\del{\Lambda_J-V_{t,J}}\psi_{1:J}(x)
    \end{align*}
Since
    \begin{align*}
    V_{t,J}^{-1}
    =  \Lambda_J^{-1} + \sigma_\xi^{-2}\Psi_{t,J}^\top\Psi_{t,J}
    \succeq  \Lambda_J^{-1},
    \end{align*}
we have $V_t\preceq\Lambda_J$. Hence,
    \begin{align*}
    0\le v_{Q_t^{(J)}}(x)\le \sK_0(x,x)
    \end{align*}
Finally, \cref{ass:gp-true} implies
    \begin{align*}
    \sup_{x\in\mathcal X}\sK_0(x,x)
    \le  B_\psi^2\sum_{j\ge1}\lambda_j  <\infty.
\end{align*}
The conclusion follows.
\end{proof}

\begin{proof}[Proof of \cref{thm:gp-fast}]
With $J\gtrsim J_t^\dag$, Substitution of the upper bounds in \cref{thm:gp-statistical,thm:sparse-gp-accuracy} and \cref{lem:gp-pointvar} into \cref{thm:gp-predictive-comparison} gives
    \begin{align*}
    \fR\del[1]{(Q_t^{(J)})_{t=0}^{T-1}}
    &\lesssim \sum_{t=1}^T\{ \varepsilon_t\alpha_t+\alpha_t^2+ \varepsilon_t^2\}\\
    &\lesssim  \sum_{t=1}^T t^{-2\fs/(d+2\fs)}
    \lesssim T^{d/(d+2\fs)},
    \end{align*}
which is the desired result.
\end{proof}

\bibliographystyle{plainnat}
\bibliography{references}

@inproceedings{vovk1990aggregating,
  title={Aggregating strategies},
  author={Vovk, Volodimir G},
  booktitle={Proceedings of the third annual workshop on Computational learning theory},
  pages={371--386},
  year={1990}
}

@inproceedings{alquier2021non,
  title={Non-exponentially weighted aggregation: regret bounds for unbounded loss functions},
  author={Alquier, Pierre},
  booktitle={International Conference on Machine Learning},
  pages={207--218},
  year={2021},
  organization={PMLR}
}

@article{vanerven2015fast,
  author  = {Tim van Erven and Peter D. Gr{{\"u}}nwald and Nishant A. Mehta and Mark D. Reid and Robert C. Williamson},
  title   = {Fast Rates in Statistical and Online Learning},
  journal = {Journal of Machine Learning Research},
  year    = {2015},
  volume  = {16},
  number  = {54},
  pages   = {1793--1861}
}

@article{dalalyan2017theoretical,
  title={Theoretical guarantees for approximate sampling from smooth and log-concave densities},
  author={Dalalyan, Arnak S},
  journal={Journal of the Royal Statistical Society Series B: Statistical Methodology},
  volume={79},
  number={3},
  pages={651--676},
  year={2017},
  publisher={Oxford University Press}
}

@article{durmus2017nonasymptotic,
  title = {Nonasymptotic convergence analysis for the unadjusted Langevin algorithm},
  author = {Alain Durmus and {\'E}ric Moulines},
  year = {2017},
  journal = {The Annals of Applied Probability},
  volume = {27},
  number = {3},
  pages = {1551--1587},
  doi = {10.1214/16-AAP1238}
}

@inproceedings{lee2021structured,
  title = {Structured logconcave sampling with a restricted Gaussian oracle},
  author = {Yin Tat Lee and Ruoqi Shen and Kevin Tian},
  year = {2021},
  booktitle = {Conference on Learning Theory},
  pages = {2993--3050}
}

@article{zhao2000bayesian,
  title = {Bayesian aspects of some nonparametric problems},
  author = {Linda H. Zhao},
  year = {2000},
  journal = {The Annals of Statistics},
  volume = {28},
  number = {2},
  pages = {532--552},
  doi = {10.1214/aos/1016218229}
}

@inproceedings{broderick2013streaming,
  author    = {Broderick, Tamara and Boyd, Nicholas and Wibisono, Andre and Wilson, Ashia C. and Jordan, Michael I.},
  title     = {Streaming Variational {Bayes}},
  booktitle = {Advances in Neural Information Processing Systems 26},
  pages     = {1727--1735},
  publisher = {Curran Associates, Inc.},
  year      = {2013}
}

@inproceedings{brosse2017compact,
  author    = {Brosse, Nicolas and Durmus, Alain and Moulines, {\'E}ric and Pereyra, Marcelo},
  title     = {Sampling from a Log-Concave Distribution with Compact Support with Proximal Langevin Monte Carlo},
  booktitle = {Proceedings of the 2017 Conference on Learning Theory},
  series    = {Proceedings of Machine Learning Research},
  volume    = {65},
  pages     = {319--342},
  publisher = {PMLR},
  year      = {2017}
}

@article{bubeck2018projected,
  author  = {Bubeck, S{\'e}bastien and Eldan, Ronen and Lehec, Joseph},
  title   = {Sampling from a Log-Concave Distribution with Projected Langevin Monte Carlo},
  journal = {Discrete \& Computational Geometry},
  volume  = {59},
  number  = {4},
  pages   = {757--783},
  year    = {2018},
  doi     = {10.1007/s00454-018-9992-1}
}

@article{chandrasekaran2013computational,
  title = {Computational and statistical tradeoffs via convex relaxation},
  author = {Chandrasekaran, Venkat and Jordan, Michael I.},
  year = {2013},
  journal = {Proceedings of the National Academy of Sciences},
  volume = {110},
  number = {13},
  pages = {E1181-E1190},
  doi = {10.1073/pnas.1302293110}
}

@inproceedings{rudi2015less,
  title = {Less is more: Nystr{\"o}m computational regularization},
  author = {Rudi, Alessandro and Camoriano, Raffaello and Rosasco, Lorenzo},
  year = {2015},
  booktitle = {Advances in Neural Information Processing Systems},
  volume = {28}
}

@article{raskutti2014early,
  title={Early stopping and non-parametric regression: an optimal data-dependent stopping rule},
  author={Raskutti, Garvesh and Wainwright, Martin J and Yu, Bin},
  journal={Journal of Machine Learning Research},
  volume={15},
  number={1},
  pages={335--366},
  year={2014},
  publisher={Journal of Machine Learning Research}
}

@book{cesabianchi2006prediction,
  author    = {Cesa-Bianchi, Nicol\`o and Lugosi, G\'abor},
  title     = {Prediction, Learning, and Games},
  publisher = {Cambridge University Press},
  year      = {2006},
  isbn      = {9780521841085}
}

@inproceedings{cherief2019onlinevi,
  author    = {Ch{\'e}rief-Abdellatif, Badr-Eddine and Alquier, Pierre and Khan, Mohammad Emtiyaz},
  title     = {A Generalization Bound for Online Variational Inference},
  booktitle = {Proceedings of The Eleventh Asian Conference on Machine Learning},
  series    = {Proceedings of Machine Learning Research},
  volume    = {101},
  pages     = {662--677},
  publisher = {PMLR},
  year      = {2019}
}

@article{clarke1990bayes,
  author  = {Clarke, Bertrand S. and Barron, Andrew R.},
  title   = {Information-Theoretic Asymptotics of {Bayes} Methods},
  journal = {IEEE Transactions on Information Theory},
  volume  = {36},
  number  = {3},
  pages   = {453--471},
  year    = {1990},
  doi     = {10.1109/18.54897}
}

@inproceedings{dai2016particle,
  author    = {Dai, Bo and He, Niao and Dai, Hanjun and Song, Le},
  title     = {Provable {Bayesian} Inference via Particle Mirror Descent},
  booktitle = {Proceedings of the 19th International Conference on Artificial Intelligence and Statistics},
  series    = {Proceedings of Machine Learning Research},
  volume    = {51},
  pages     = {985--994},
  publisher = {PMLR},
  year      = {2016}
}

@article{delmoral2006smc,
  author  = {Del Moral, Pierre and Doucet, Arnaud and Jasra, Ajay},
  title   = {Sequential {Monte Carlo} Samplers},
  journal = {Journal of the Royal Statistical Society Series B: Statistical Methodology},
  volume  = {68},
  number  = {3},
  pages   = {411--436},
  year    = {2006},
  doi     = {10.1111/j.1467-9868.2006.00553.x}
}

@article{diaconis1979conjugate,
  title = {Conjugate priors for exponential families},
  author = {Persi Diaconis and Donald Ylvisaker},
  year = {1979},
  journal = {The Annals of Statistics},
  volume = {7},
  number = {2},
  pages = {269--281}
}

@article{hazan2007logarithmic,
  author  = {Hazan, Elad and Agarwal, Amit and Kale, Satyen},
  title   = {Logarithmic Regret Algorithms for Online Convex Optimization},
  journal = {Machine Learning},
  volume  = {69},
  number  = {2--3},
  pages   = {169--192},
  year    = {2007},
  doi     = {10.1007/s10994-007-5016-8}
}

@article{jun2026adaptive,
  title = {Adaptive Bayesian online learning via expert aggregation},
  author = {Jun, Jungbin and Ohn, Ilsang},
  year = {2026},
  journal = {arXiv preprint},
  eprint = {2607.20239},
  archivePrefix = {arXiv}
}

@inproceedings{jones2024bong,
  author    = {Jones, Matt and Chang, Peter and Murphy, Kevin},
  title     = {{Bayesian} Online Natural Gradient ({BONG})},
  booktitle = {Advances in Neural Information Processing Systems 37},
  pages     = {131104--131153},
  year      = {2024},
  doi       = {10.52202/079017-4167}
}

@inproceedings{kakade2004online,
  author    = {Kakade, Sham M. and Ng, Andrew Y.},
  title     = {Online Bounds for {Bayesian} Algorithms},
  booktitle = {Advances in Neural Information Processing Systems 17},
  pages     = {641--648},
  year      = {2004}
}

@article{khan2023bayesian,
  title = {The Bayesian learning rule},
  author = {Mohammad Emtiyaz Khan and H{\aa}vard Rue},
  year = {2023},
  journal = {Journal of Machine Learning Research},
  volume = {24},
  number = {281},
  pages = {1--46}
}

@article{spantini2015optimal,
  title={Optimal low-rank approximations of Bayesian linear inverse problems},
  author={Spantini, Alessio and Solonen, Antti and Cui, Tiangang and Martin, James and Tenorio, Luis and Marzouk, Youssef},
  journal={SIAM Journal on Scientific Computing},
  volume={37},
  number={6},
  pages={A2451--A2487},
  year={2015},
  publisher={SIAM}
}

@inproceedings{lee2019online,
  author    = {Lee, Holden and Mangoubi, Oren and Vishnoi, Nisheeth K.},
  title     = {Online Sampling from Log-Concave Distributions},
  booktitle = {Advances in Neural Information Processing Systems 32},
  year      = {2019}
}

@book{tsybakov2009nonparametric,
  author    = {Tsybakov, Alexandre B.},
  title     = {Introduction to Nonparametric Estimation},
  series    = {Springer Series in Statistics},
  publisher = {Springer},
  address   = {New York},
  year      = {2009},
  doi       = {10.1007/b13794}
}

@inproceedings{vanderhoeven2018manyfaces,
  author    = {van der Hoeven, Dirk and van Erven, Tim and Kot{\l}owski, Wojciech},
  title     = {The Many Faces of Exponential Weights in Online Learning},
  booktitle = {Proceedings of the 31st Conference on Learning Theory},
  series    = {Proceedings of Machine Learning Research},
  volume    = {75},
  pages     = {2067--2092},
  publisher = {PMLR},
  year      = {2018}
}

@article{xie2000minimax,
  author  = {Xie, Qun and Barron, Andrew R.},
  title   = {Asymptotic minimax regret for data compression, gambling, and prediction},
  journal = {IEEE Transactions on Information Theory},
  volume  = {46},
  number  = {2},
  pages   = {431--445},
  year    = {2000},
  doi     = {10.1109/18.825803}
}

@article{yang2023particle,
  author  = {Yang, Yifan and Liu, Chang and Zhang, Zheng},
  title   = {Particle-based Online {Bayesian} Sampling},
  journal = {arXiv preprint},
  year    = {2023},
  eprint  = {2302.14796},
  archivePrefix = {arXiv},
  doi     = {10.48550/arXiv.2302.14796}
}

@article{alquier2016properties,
  title={On the properties of variational approximations of Gibbs posteriors},
  author={Alquier, Pierre and Ridgway, James and Chopin, Nicolas},
  journal={Journal of Machine Learning Research},
  volume={17},
  number={1},
  pages={8374--8414},
  year={2016},
  publisher={Journal of Machine Learning Research}
}

@article{zhang2020convergence,
  title={Convergence rates of variational posterior distributions},
  author={Zhang, Fengshuo and Gao, Chao},
  journal={The Annals of Statistics},
  volume={48},
  number={4},
  pages={2180 -- 2207},
  year={2020}
}

@article{alquier2020concentration,
  title={Concentration of tempered posteriors and of their variational approximations},
  author={Alquier, Pierre and Ridgway, James},
  journal={The Annals of Statistics},
  volume={48},
  number={3},
  pages={1475--1497},
  year={2020},
  publisher={Institute of Mathematical Statistics}
}

@article{ohn2024adaptive,
  title={Adaptive variational {Bayes}: {O}ptimality, computation and applications},
  author={Ohn, Ilsang and Lin, Lizhen},
  journal={The Annals of Statistics},
  volume={52},
  number={1},
  pages={335--363},
  year={2024},
  publisher={Institute of Mathematical Statistics}
}

@inproceedings{titsias2009variational,
  author    = {Titsias, Michalis K.},
  title     = {Variational Learning of Inducing Variables in Sparse Gaussian Processes},
  booktitle = {Proceedings of the Twelfth International Conference on Artificial Intelligence and Statistics},
  series    = {Proceedings of Machine Learning Research},
  volume    = {5},
  pages     = {567--574},
  year      = {2009}
}

@inproceedings{matthews2016sparse,
  author    = {Matthews, Alexander G. de G. and Hensman, James and Turner, Richard E. and Ghahramani, Zoubin},
  title     = {On Sparse Variational Methods and the Kullback--Leibler Divergence between Stochastic Processes},
  booktitle = {Proceedings of the 19th International Conference on Artificial Intelligence and Statistics},
  series    = {Proceedings of Machine Learning Research},
  volume    = {51},
  pages     = {231--239},
  year      = {2016}
}

@inproceedings{burt2019sparsegp,
  author    = {Burt, David and Rasmussen, Carl Edward and van der Wilk, Mark},
  title     = {Rates of Convergence for Sparse Variational Gaussian Process Regression},
  booktitle = {Proceedings of the 36th International Conference on Machine Learning},
  series    = {Proceedings of Machine Learning Research},
  volume    = {97},
  pages     = {862--871},
  year      = {2019}
}

@article{burt2020convergence,
  title = {Convergence of sparse variational inference in Gaussian processes regression},
  author = {David Burt and Carl Edward Rasmussen and Mark van der Wilk},
  year = {2020},
  journal = {Journal of Machine Learning Research},
  volume = {21},
  number = {131},
  pages = {1--63}
}

@article{nieman2022contraction,
  title = {Contraction rates for sparse variational approximations in Gaussian process regression},
  author = {Dennis Nieman and Botond Szabo and Harry van Zanten},
  year = {2022},
  journal = {Journal of Machine Learning Research},
  volume = {23},
  number = {205},
  pages = {1--26}
}

@article{nieman2023uncertainty,
  title = {Uncertainty quantification for sparse spectral variational approximations in Gaussian process regression},
  author = {Dennis Nieman and Botond Szabo and Harry van Zanten},
  year = {2023},
  journal = {Electronic Journal of Statistics},
  volume = {17},
  number = {2},
  pages = {2250--2288},
  doi = {10.1214/23-EJS2155}
}

@article{nieman2025adaptive,
  title={Adaptive sparse variational approximations for Gaussian process regression},
  author={Nieman, Dennis and Szab{\'o}, Botond},
  journal={Bayesian Analysis},
  year={2025},
  publisher={International Society for Bayesian Analysis}
}

@article{hensman2018variational,
  title = {Variational fourier features for Gaussian processes},
  author = {James Hensman and Nicolas Durrande and Arno Solin},
  year = {2018},
  journal = {Journal of Machine Learning Research},
  volume = {18},
  number = {151},
  pages = {1--52}
}

@inproceedings{bui2017streaming,
  author    = {Bui, Thang D. and Nguyen, Cuong and Turner, Richard E.},
  title     = {Streaming Sparse Gaussian Process Approximations},
  booktitle = {Advances in Neural Information Processing Systems 30},
  pages     = {3299--3307},
  year      = {2017}
}

@article{fischer2020sobolev,
  title = {Sobolev norm learning rates for regularized least-squares algorithms},
  author = {Simon Fischer and Ingo Steinwart},
  year = {2020},
  journal = {Journal of Machine Learning Research},
  volume = {21},
  number = {205},
  pages = {1--38}
}

@article{talagrand1996transportation,
  author  = {Talagrand, Michel},
  title   = {Transportation Cost for Gaussian and Other Product Measures},
  journal = {Geometric and Functional Analysis},
  volume  = {6},
  number  = {3},
  pages   = {587--600},
  year    = {1996}
}

@article{lee2026online,
  title = {Online Bernstein-von Mises theorem},
  author = {Jeyong Lee and Junhyeok Choi and Minwoo Chae},
  year = {2026},
  journal = {Journal of Machine Learning Research},
  volume = {27},
  number = {49},
  pages = {1--124}
}

@article{lee2026bayesian,
  title={Bayesian online learning in the one-pass regime: Frequentist validity and uncertainty quantification},
  author={Lee, Jeyong and Choi, Junhyeok and Kim, Dongguen and Chae, Minwoo},
  journal={arXiv preprint},
  year={2026},
  eprint={2604.27442},
  archivePrefix={arXiv}
}

\end{document}